%% file: ms.tex
\documentclass{article}

\usepackage[accepted]{icml2019nc}

\usepackage{microtype, graphicx, booktabs,verbatim, amsmath, amssymb, subfig, bbm, mathrsfs, amsthm, url,adjustbox,mathtools, xspace}
\usepackage{textcomp}
\usepackage{cleveref}
\usepackage[normalem]{ulem}
\usepackage[framemethod=TikZ]{mdframed}

\usepackage{enumerate, algorithm, algorithmic, pifont, csquotes, dashrule, tikz, bm}
\usepackage[title]{appendix}

\input{commands.tex}
\newenvironment{example}{\begin{proof}}{\end{proof}}

\newcommand{\mimoalg}{\textsc{A-MIMO}}
\newcommand{\momialg}{\textsc{A-MOMI}}

\newcommand{\momichoose}{\ensuremath{\textsc{MOMI}_{gen}}\xspace}
\newcommand{\momi}{\textsc{MOMI}\xspace}
\newcommand{\poly}{\text{poly}\xspace}
\newcommand{\OPT}{\text{OPT}\xspace}
\newcommand{\MR}{\text{MR}\xspace}
\newcommand{\cost}{\text{cost}\xspace}
\newcommand{\NP}{\text{\it NP}}
\newcommand{\DTIME}{\text{\it DTIME}}
\newcommand{\PTIME}{\text{\it P}}

\newcommand{\namecite}[1]{\citeauthor{#1}~(\citeyear{#1})}

\begin{document}

\twocolumn[
\icmltitle{Finding Options that Minimize Planning Time}
\icmlkeywords{Planning, Reinforcement Learning}

% --- Author Info ---
\begin{icmlauthorlist}
\icmlauthor{Yuu Jinnai}{b}
\icmlauthor{David Abel}{b}
\icmlauthor{D Ellis Hershkowitz}{c}
\icmlauthor{Michael L. Littman}{b}
\icmlauthor{George Konidaris}{b}
\end{icmlauthorlist}

\icmlaffiliation{b}{Brown University, Providence, RI, United States}
\icmlaffiliation{c}{Carnegie Mellon University, Pittsburgh, PA, United States}

\icmlcorrespondingauthor{Yuu Jinnai}{yuu\_jinnai@brown.edu}

\vskip 0.3in
]

\printAffiliationsAndNotice{}

% -- Abstract --
\begin{abstract}
We formalize the problem of selecting the optimal set of options for planning as that of  computing the smallest set of options so that planning converges in less than a given maximum of value-iteration passes.
We first show that the problem is  $\NP$-hard, even if the task is constrained to be deterministic---the first such complexity result for option discovery.
We then present the first polynomial-time boundedly suboptimal approximation algorithm for this setting, and empirically evaluate it against both the optimal options and a representative collection of heuristic approaches in simple grid-based domains including the classic four-rooms problem.
\end{abstract}

% ------------------
% -- Introduction --
% ------------------
\section{Introduction}

% options can help.
Markov Decision Processes or MDPs~\cite{puterman1994markov} are an expressive yet simple model of sequential decision-making environments. However, MDPs are computationally expensive to solve~\cite{papadimitriou87,littman97b,goldsmith97b}.
% papadimitriou87 shows P-complete for flat representation
% littman97b shows EXP-completness for STRIPS representations and PSPACE-completeness if number of operators is limited
% littman95e talks about algorithms for flat spaces
% goldsmith97b summarizes all previous results and provides new ones like NP^PP for succinct plans
One approach to solving such problems is to add high-level, temporally extended actions---often formalized as options~\cite{sutton1999between}---to the action space. The right set of options allows planning to probe more deeply into the search space with a single computation. Thus, if options are chosen appropriately, planning algorithms can find good plans with less computation.

% Prior art.
Indeed, previous work has offered substantial support that abstract actions can accelerate planning~\cite{mann2014scaling}. However, little is known about how to find the right set of options to use for planning. Prior work often seeks to codify an intuitive notion of what underlies an effective option, such as identifying relatively novel states \cite{Simsek04}, identifying bottleneck states or high-betweenness states \cite{Simsek2005,csimcsek2009skill,bacon2013bottleneck,moradi2012automatic}, finding repeated policy fragments \cite{pickett2002policyblocks}, or finding states that often occur on successful trajectories \cite{mcgovern2001automatic,bakker2004hierarchical}.
While such intuitions often capture important aspects of the role of options in planning, the resulting algorithms are somewhat heuristic in that {\it they are not based on optimizing any precise performance-related metric}; consequently, their relative performance can only be evaluated empirically.

% dnote: practical algorithm? is that what we're after?
% Our objective.
We aim to formalize what it means to find the set of options that is optimal for planning, and to use the resulting formalization to develop an approximation algorithm with a principled theoretical foundation.
Specifically, we consider the problem of finding the smallest set of options so that planning converges in fewer than a given maximum of $\ell$ iterations of the planning algorithm, value iteration (VI).
Our main result shows that this problem is $\NP$-hard. More precisely, the problem:
\begin{enumerate}
    \item is $2^{\log^{1 - \epsilon} n}$-hard to approximate for any $\epsilon > 0$ unless $\NP \subseteq \DTIME(n^{\poly \log n})$\footnote{This is a standard complexity assumption: See, for example, \citet{dinitz2012label}}, where $n$ is the input size;
    \item is $\Omega(\log n)$-hard to approximate even for deterministic MDPs unless $\PTIME = \NP$;
    \item has a $O(n)$-approximation algorithm;
    \item has a $O(\log n)$-approximation algorithm for deterministic MDPs.
\end{enumerate}

In Section \ref{sec:algorithms}, we show \momialg{}, a polynomial-time approximation algorithm that has $O(n)$ suboptimality in general and $O(\log n)$ suboptimality for deterministic MDPs.
The expression $2^{\log^{1-\epsilon}n}$ is only slightly smaller than $n$: if $\epsilon=0$ then $\Omega(2^{\log n}) = \Omega(n)$. Thus, the inapproximability results claim that \momialg{} is close to the best possible approximation factor.

In addition, we consider the complementary problem of finding a set of $k$ options that minimize the number of VI iterations until convergence. We show that this problem is also $\NP$-hard, even for a deterministic MDP.

% Specifically, we consider two settings that describe the problem of finding the right set of options: 1) computing the smallest set of options so that planning converges in less than a given maximum of $\ell$ VI iterations and 2) finding the set of $k$ options that minimize the number of VI iterations until convergence.
% We show that both problems are NP-hard, even for a deterministic MDP, and therefore harder than directly solving the MDP, which takes polynomial time \cite{littman1995complexity}. 
% We then provide a polynomial-time approximation algorithm for a subclass of each problem, \momialg{} and \mimoalg{}, that computes approximately optimal options for MDPs with bounded return and goal states. We prove that both algorithms have bounded suboptimality for deterministic tasks. 
% These algorithms are not practical for speeding up run-time performance, as they are computationally harder than solving the MDP itself. 
% The purpose of the algorithm is to analyze and evaluate the utility of options generated by heuristic methods.
Finally, we empirically evaluate the performance of two heuristic approaches for option discovery, betweenness options \cite{csimcsek2009skill} and Eigenoptions \cite{machado2018laplacian}, against the proposed approximation algorithms and the optimal options in standard grid domains.

% We show that MOMI is NP-hard to approximate.
% Even when the MDP is deterministic and set of point options are given, MOMI is O(log n)-hard to approximate.
% We show that there exists an O(log n) algorithm for such setting.

% --- option type table ---
% \begin{table*}[htb]
%     \centering
%     \begin{tabular}{c|c|c}
%     & $V^* - V_n$ & $n=$\#iterations without options  \\ \hline
%     Deterministic        & Exact    & $n = \max_{s \in S} d(s, g)$ ($d=min\_cost$) \\
%     Stochastic (no loop) & Exact    & $n = \max_{s \in S} d(s, g)$ ($d=minmax\_cost$) \\
%     Stochastic (loop)    & $\epsilon$-optimal  & $n \leq \max_{s \in S} d(s, g)$ ($d=VI iteration$) \\
%     \end{tabular}
%     \caption{Optimal point options to minimize the number of VI iterations for goal-based MDP.}
%     \label{tab:summary}
% \end{table*}

% ----------------
% -- Background --
% ----------------
\section{Background}
\label{sec:background}

We first provide background on Markov Decision Processes (MDPs), planning, and options. %All caligraphic font symbols ($\mc{S}, \mc{O}$), denote sets.

% --- Markov Decision Processes ---
\subsection{Markov Decision Processes and Planning}

% MDPs.
An MDP is a five tuple: $\langle \mc{S}, \mc{A}, R, T, \gamma \rangle$, where
$\mc{S}$ is a finite set of states; $\mc{A}$ is a finite set of actions; $R: \mc{S} \times \mc{A} \ra [0, \textsc{RMax}]$ is a reward function; $T: \mc{S} \times \mc{A} \ra \Pr(\mc{S})$ is a transition function, denoting the probability of arriving in state $s' \in \mc{S}$ after executing action $a \in \mc{A}$ in state $s \in \mc{S}$; and $\gamma \in [0, 1]$ is a discount factor, expressing the agent's preference for immediate over delayed rewards.

% Policies and values.
An action-selection strategy is modeled by a 
{\it policy}, $\pi : \mc{S} \ra \Pr(\mc{A})$, mapping
states to a distribution over actions. Typically, the goal of planning in an MDP is to {\it solve} the MDP---that is, to compute an optimal policy. A policy $\pi$ is evaluated according to the Bellman equation, denoting the long term expected reward received by executing $\pi$:
\begin{equation}
    V^\pi(s) = R(s,\pi(s)) + \gamma \sum_{s' \in \mc{S}} T(s,\pi(s),s') V^\pi(s').
\end{equation}
We denote $\pi^*(s) = \argmax_\pi V^\pi(s)$ and $V^*(s) = \max_\pi V^\pi(s)$ as the optimal policy and value function, respectively. % \mnote{max over policies is not defined. maybe include (s)?}

% --- Planning ---
%\subsection{Planning}

%\dnote{I'm not sure if planning needs a subsection if it only contains 1 sentence and a definition. Consider removing the subsection header?}

The core problem we study is {\it planning}, namely, computing a near optimal policy for a given MDP. The main variant of the planning problem we study we denote the {\it value-planning problem}:
% Value Planning Problem
\ddef{Value-Planning Problem}{
    {\bf Given} an MDP $M = \langle \mc{S}, \mc{A}, R, T, \gamma \rangle$ and a non-negative real-value $\epsilon$, {\bf return} a value function, $V$ such that $|V(s) - V^*(s)| < \epsilon$ for all $s \in \mc{S}$.}

The value-planning problem can be solved in time polynomial in the size of the state space.

% --- options ---
\subsection{Options and Value Iteration}

Temporally extended actions offer great potential for mitigating the difficulty of solving complex MDPs, either through planning or reinforcement learning \cite{sutton1999between}. Indeed, it is possible that options that are useful for learning are not necessarily useful for planning, and vice versa. Identifying techniques that produce good options in these scenarios is an important open problem in the literature.

We use the standard definition of options \cite{sutton1999between}:
% --- option Definition ---
\ddef{option}{An option $o$ is defined by a triple: $(\mc{I}, \pi, \beta)$ where:
\begin{itemize}
    \item $\mc{I} \subseteq \mc{S}$ is a set of states where the option can initiate,
    \item $\pi : \mc{S} \ra \Pr(\mc{A})$ is a policy,
    \item $\beta : \mc{S} \ra [0, 1]$, is a termination condition.
\end{itemize}
We let $\mc{O}_{all}$ denote the set containing all options.
}

% Multi-time model, SMDPs.
In planning, options have a well defined transition and reward model for each state named the multi-time model, introduced by~\namecite{precup1998multi}:
\begin{align}
    T_\gamma(s,o,s') &= \sum_{t=0}^\infty \gamma^t \Pr(s_t = s', \beta(s_t) \mid s, o). \\
    R_\gamma(s,o) &= \underset{o_\pi}{\bE}\left[r_1 + \gamma r_2 + \ldots + \gamma^{k-1} r_k \bigmid s, o\right].
\end{align}
We use the multi-time model for value iteration. The algorithm computes a sequence of functions $V_0, V_1, ..., V_b$ using the Bellman optimality operator on the multi-time model:
\begin{equation}
    V_{i+1} (s) = \max_{o \in A \cup \mc{O}} \left(R_\gamma(s,o) + \sum_{s' \in S} T_\gamma(s,o,s') V_{i}(s') \right).
\end{equation}
The problem we consider is to find a set of options to add to the set of primitive actions that minimize the number of iterations required for VI to converge:\footnote{We can ensure $|V^*(s) - V_{i}(s)| < \epsilon$ by running VI until $|V_{i+1}(s) - V_{i}(s)| < \epsilon (1 - \gamma) / 2 \gamma$ for all $s \in \mc{S}$ \cite{williams1993tight}.}

\ddef{$L_{\epsilon, V_0}(\mc{O})$}{
The number of iterations $L_{\epsilon, V_0}(\mc{O})$ of VI using the joint action set $\mc{A} \cup \mc{O}$, with $\mc{O}$ a non-empty set of options, is the smallest $b$ at which $|V_b'(s) - V^*(s)| < \epsilon$ for all $s \in \mc{S}$, $b' \geq b$. %where $\epsilon$ is a tolerance to the suboptimality of the plan quality.
}

% --- Point options ---
\subsubsection{Point options.}

The options formalism is immensely general. Due to its generality, a single option can actually encode several completely unrelated sets of different behaviors. Consider the nine-state example MDP pictured in Figure~\ref{fig:point_option_example}; a single option can in fact initiate, make decisions in, and terminate along entirely independent trajectories. As we consider more complex MDPs (which, as discussed earlier, is often a motivation for introducing options), the number of independent behaviors that can be encoded by a single option increases further still.

% Point option counter example.
\begin{figure}
    \centering
    \includegraphics[width=0.5\columnwidth]{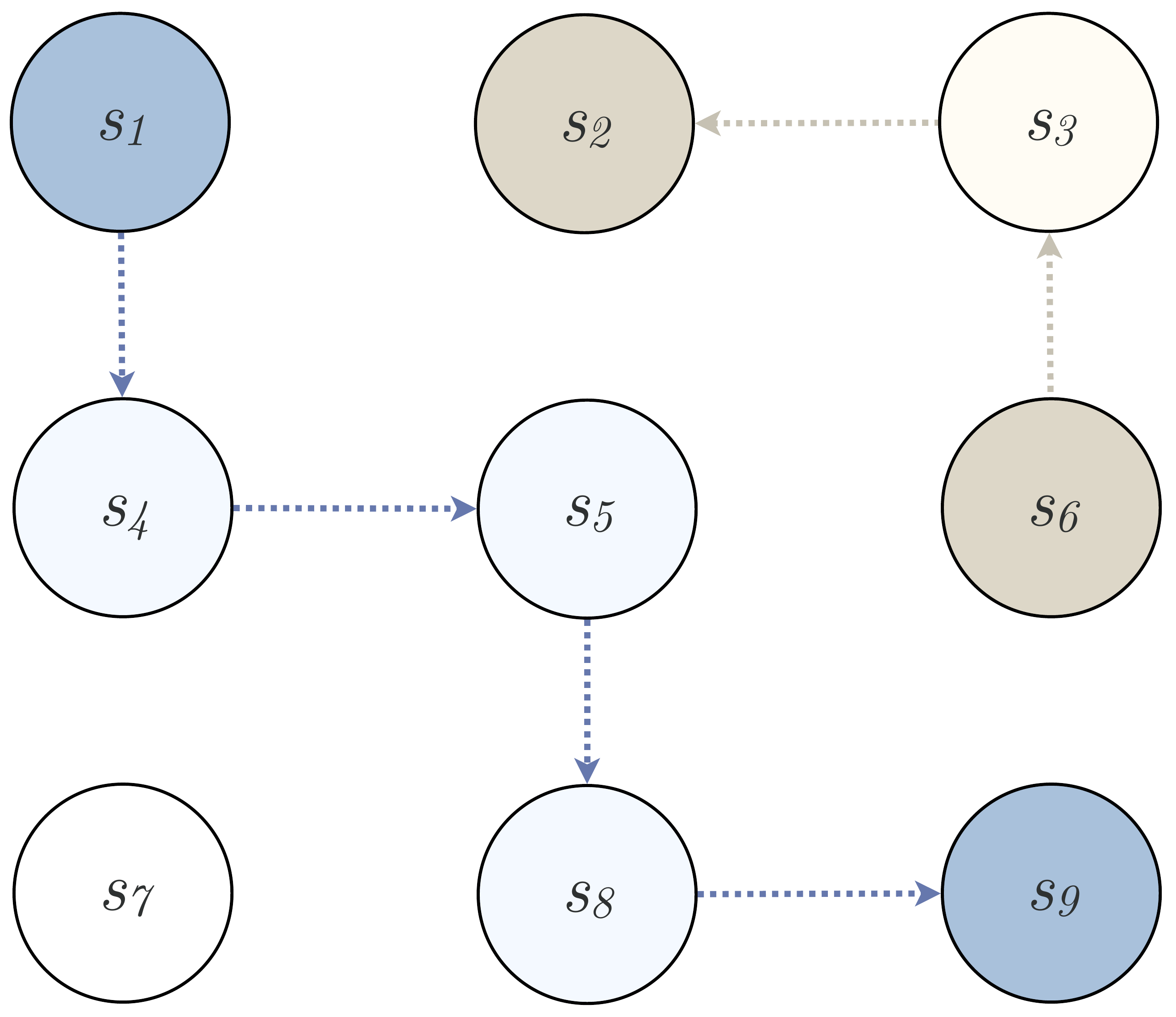}
    \caption{A single option can encode multiple unrelated behaviors. The dark circles indicate where the option can be initiated ($s_1$ \& $s_6$) and terminated ($s_2$ \& $s_9$), whereas the lighter circles denote the states visited by the option policy when applied in the respective initiating state.}
    \label{fig:point_option_example}
\end{figure}

As a result, it can be difficult to reason about the impact of adding a single option, in the traditional sense. As the MDP grows larger, a combinatorial number of different behaviors can emerge from ``one" option. Consequently, it is difficult to address the question: which {\it single} option helps planning the most? As MDPs grow large, one option can encode a large number of possible, independent behaviors. Thus, we instead introduce and study ``point options", which only allow for a single continuous stream of behavior:
% --- Point option Definition ---
\ddef{Point option}{
    A {\bf point option} is any option whose initiation set and termination set are each true for exactly one state each:
\begin{align}
        |\{s \in \mc{S} : \mc{I}(s) = 1\}| &= 1, \\
        |\{s \in \mc{S} : \beta(s) > 0\}| &= 1, \\
        |\{s \in \mc{S} : \beta(s) = 1\}| &= 1.
\end{align}

    We let $\mc{O}_p$ denote the set containing all point options.
}
For simplicity, we denote the initiation state as $\mc{I}_o$ and the termination state as $\beta_o$ for a point option $o$.

% options for planning.
To plan with a point option from state $s$, the agent runs value iteration using a model $Q(s, o) = R(s, o) + \gamma^k V(s')$ in addition to the backup operations by primitive actions where $k$ is the duration of the option. We assume that the model of each option is given to the agent and ignore the computation cost for computing the model for the options.

Point options are a useful subclass to consider for several reasons. First, a point option is a simple model for a temporally extended action. Second, the policy of the point option can be calculated as a path-planning problem for deterministic MDPs. % \gdknote{Only for deterministic MDPs.} 
Third, any other options with a single termination state with termination probability 1 can be represented as a collection of point options. Fourth, a point option has a fixed amount of computational overhead per iteration.

% ------------------------
% -- Complexity Results --
% ------------------------
\section{Complexity Results}
\label{sec:complexity}

Our main results focus on two computational problems:
\begin{enumerate}
    \item \textsc{MinOptionMaxIter} (MOMI): Which set of options lets
    value iteration converge in at most $\ell$ iterations?
    \item \textsc{MinIterMaxOption} (MIMO): Which set of $k$ or fewer options minimizes the number of iterations to convergence?
\end{enumerate}

More formally, MOMI is defined as follows.

% --- MOMI Problem ---
\ddef{MOMI}{The \textsc{MinOptionMaxIter} problem: \\
    {\bf Given} an MDP $M$, a non-negative real-value $\epsilon$, an initial value function $V_0$, and an integer $\ell$ {\bf return} $\mc{O}$ that minimizes $|\mc{O}|$ subject to $\mc{O} \subseteq \mc{O}_p$ and $L_{\epsilon, V_0}(\mc{O}) \leq \ell$.
}

We then consider the complementary optimization problem: % the inverse problem: \ynote{Is this a correct term? I think 'inverse problem' is used more often when some of the inputs and the outputs are swapped. Here we are swapping the optimization criteria and the constraint on the output.}  \gdknote{Usually it is ``dual problem'', also possibly ``complementary''.} 
compute a set of $k$ options which minimizes the number of iterations. Motivated by this scenario, the second problem we study is
\textsc{MinIterMaxOption} (MIMO).

% --- MIMO Problem ---
\ddef{MIMO}{The \textsc{MinIterMaxOption} problem: \\
    {\bf Given} an MDP $M$, a non-negative real-value $\epsilon$, an initial value function $V_0$, and an integer $k$ {\bf return} $\mc{O}$ that minimizes $L_{\epsilon, V_0}(\mc{O})$, subject to $\mc{O} \subseteq \mc{O}_p$ and $|\mc{O}| \leq k$.
}
%%%%%%%%

We now introduce our main result, which shows that both MOMI and MIMO are $\NP$-hard.
% --- Theorem: MIMO and MOMI are NP-Hard ---
\begin{theorem}
    MOMI and MIMO are $\NP$-hard.
    \label{thm:main_result}
\end{theorem}

% MIMO/MOMI NP-Hard Proof.
\input{proofs/mimo_momi_np_hard.tex}
% ------------------------

% Discussion of the NP-Hard result.
% Note that as a special case where $k=1$, MIMO can be solved in polynomial time. A simple procedure is to calculate $L_{\epsilon, V_0}(\mc{O})$ for every possible point options and pick the best one. As $L_{\epsilon, V_0}(\mc{O})$ can be calculated within polynomial time for a given $\mc{O}$ and the number of possible point options is at most a square of the size of the state set, we can enumerate every $L_{\epsilon, V_0}(\mc{O})$ with $|\mc{O}| = 1$ and find the best $\mc{O}$ within polynomial time.

%That being said, MIMO and MOMI are NP-hard in general.

% As MOMI and MIMO are mutually reducible within polynomial time, we guess it is also NP-hard to approximate.
%\ynote{I am not able to show any inapproximability for MIMO.}

% Other variants.
%MIMO and MOMI are not the only two possible formulations of interest.
%For example, one can think of simultaneously optimizing the number of options and the number of iterations. That is, given an MDP, minimize $L_{\epsilon, V_0}(\mc{O}) + w |\mc{O}|$ where $w$ is a cost on adding an option. This is yet another NP-hard optimization problem with OI-DEC as the corresponding decision problem.

%%%%%%%%%%%%%%%%%%%%
% --- Extensions ---
\subsection{Generalizations of MOMI and MIMO}
A natural question is whether Theorem~\ref{thm:main_result} extends to more general option-construction settings. We consider two possible extensions,
which we believe offer significant coverage of finding optimal options for planning in general.
%\gdknote{Guys please check to make sure the last sentence is not too grandiose. But I really want to make the point here that these generalizations really settle the question.}
%\dnote{So, the reason I'm a bit nervous about this is the (now deleted) comment about the fully general case -- let's discuss at our meeting.}

% Extension one: O_i \subset \mc{O}
We first consider the case where the options are not necessarily point options. There is little sense in considering MOMI where one can choose any option since clearly the best option is the option whose policy is the optimal policy.  Thus, using the space of all options $\mc{O}_{all}$ we generalize MOMI as follows:
\ddef{MOMI$_{gen}$}{\\
    {\bf Given} an MDP $M$, a non-negative real-value $\epsilon$, an initial value function $V_0$, $\mc{O'} \subseteq \mc{O}_{all}$, and an integer $\ell$, {\bf return} $\mc{O}$ minimizing $|\mc{O}|$ subject to $L_{\epsilon, V_0}(\mc{O}) \leq \ell$ and $\mc{O} \subseteq \mc{O'}$.
}

%\momichoose is a natural generalization of MOMI beyond point options. Rather with \momichoose we model the scenario where we have a set of abstractions that we have learned from past experiences and we would like to choose the learned abstractions which are the best for our current problem.

\begin{theorem}
    MOMI$_{gen}$ and MIMO$_{gen}$ are $\NP$-hard.
\end{theorem}

The proof follows from the fact that MOMI$_{gen}$ is a superset of MOMI and MIMO$_{gen}$ is a superset of MIMO.

% Multitask generalization.
We next consider the multi-task generalization, where
we aim to find a smallest number of options which the expected number of iterations to solve a problem $M$ sampled from a distribution of MDPs, $D$, is bounded:

\ddef{MOMI$_{multi}$}{\\ %The MIMO$_{multi}$ defines the following computational problem: \\
    {\bf Given} A distribution of MDPs $D$, $\mc{O'} \subseteq \mc{O}_{all}$, a non-negative real-value $\epsilon$, an initial value function $V_0$, and an integer $\ell$, {\bf return} $\mc{O}$ that minimizes $|\mc{O}|$ such that $E_{M \sim D}[L_M(\mc{O})] \leq \ell$ and $\mc{O} \subseteq \mc{O'}$.
}

\begin{theorem}
    MOMI$_{multi}$ and MIMO$_{multi}$ are NP-hard.
\end{theorem}

The proof follows from the fact that MOMI$_{multi}$ is a superset of MOMI$_{gen}$ and MIMO$_{multi}$ is a superset of MIMO$_{gen}$.

In light of the computational difficulty of both problems, 
the appropriate approach is to find tractable approximation algorithms.
However, even approximately solving MOMI is hard. More precisely:
% --- Theorem: MOMI Approxiamtion ---
\begin{theorem} {\ }
    \begin{enumerate}
    \itemsep0em 
        \item MOMI is $\Omega(\log n)$ hard to approximate even for deterministic MDPs unless $P = \NP$.
        \item MOMI$_{gen}$ is $2^{\log^{1-\epsilon}n}$-hard to approximate for any $\epsilon>0$ even for deterministic MDPs unless $\NP \subseteq DTIME(n^{poly \log n})$.
        \item MOMI is $2^{\log^{1-\epsilon}n}$-hard to approximate for any $\epsilon>0$ unless $\NP \subseteq DTIME(n^{poly \log n})$.
    \end{enumerate}
\label{th:MOMIHardness}
\end{theorem}

\begin{proof}
    See appendix.
\end{proof}

% $2^{\log^{1-\epsilon}n}$ in Theorem \ref{th:MOMIHardness} can be thought of as approaching polynomial hardness of approximation, slightly smaller than $n$: if $\epsilon=0$ then $\Omega(2^{\log n}) = \Omega(n)$. % already explained in intro
Note that an $O(n)$-approximation is achievable by the trivial algorithm that returns a set of all candidate options. Thus, Theorem~\ref{th:MOMIHardness} roughly states that there is no polynomial time approximation algorithms other than the trivial algorithm for MOMI. 

In the next section we show that an $O(\log n)$-approximation is achievable if the MDP is deterministic and the agent is given a set of all point options. Thus, together, these two results give a formal separation between the hardness of abstraction in MDPs with and without stochasticity.

%%%%%%%%%%%%%%%%%%
% Summary
In summary, {\it the problem of computing optimal behavioral abstractions for planning is computationally intractable}.

% ----------------
% -- Algorithms --
% ----------------
\section{Approximation Algorithms}
\label{sec:algorithms}

We now provide polynomial-time approximation algorithms, \mimoalg{} and \momialg{}, to solve MOMI and MIMO, respectively. Both algorithms have bounded suboptimality slightly worse than a constant factor for deterministic MDPs.

We assume that (1) there is exactly one absorbing state $g \in \mc{S}$ with $T(g, a, g) = 1$ and $R(g, a) = 0$, and every optimal policy eventually reaches $g$ with probability 1, (2) there is no cycle with a positive reward involved in the optimal policy's trajectory. That is, $V^\pi_{+}(s) := {\bE}[ \sum_{t=0}^{\infty} \max\{0, R(s_t, a_t)\}] < \infty$ for all policies $\pi$. Note that we can convert a problem with multiple goals to a problem with a single goal by adding a new absorbing state $g$ to the MDP and adding a transition from each of the original goals to $g$.

% Roughly speaking, the algorithm require that there is no cycle with a positive reward in the MDP.
% A GSSP MDP is a tuple $\langle \mc{S}, \mc{A}, R, T, \mc{G}, \gamma \rangle$ with $\mc{G} \subseteq \mc{S}$ representing a set of goal states. For each $g \in \mc{G}$ $T(g, a, g) = 1$ and $R(g, a) = 0$. We assume that $\mc{G}$ has exactly one state and every optimal policy reaches to one of the goal states with probability 1 and $V^\pi_{+}(s) := {\bE}[ \sum_{t=0}^{\infty} \max\{0, R(s, a)\}] < \infty$ for all policies $\pi$. $V^\pi_{+}(s)$ is the expected sum of non-negative rewards yielded by the given policy. \ynote{G can have multiple goal states in GSSP?}
% GSSP subsumes several important MDP classes defined by \citeauthor{puterman2014markov} (\citeyear{puterman2014markov}) including stochastic shortest path, positive-bounded, and negative problems.

Unfortunately, these algorithms are computationally harder than solving the MDP itself, and are thus not practical for planning. Instead, they are useful for analyzing and evaluating heuristically generated options. If the option set generated by the heuristic methods outperforms the option set found by the following algorithms, then one can claim that the option set found by the heuristic is close to the optimal option set (for that MDP). Our algorithms have a formal guarantee on bounded suboptimality if the MDP is deterministic, so any heuristic method that provably exceeds our algorithm's performance will also guarantee bounded suboptimality.
We also believe these algorithms may be a useful foundation
for future option discovery methods.

% ----------
% -- MOMI --
% ----------
\subsection{\momialg{}}

We now describe a polynomial-time approximation algorithm, \momialg{}, based on using set cover to solve MOMI.
The overview of the procedure is as follows.
\begin{enumerate}
    \item Compute an asymmetric distance function $d_\epsilon(s, s'): S \times S \rightarrow \mathbb{N}$ representing the number of iterations for a state $s$ to reach its $\epsilon$-optimal value if we add a point option from a state $s'$ to a goal state $g$.
    \item For every state $s_i$, compute a set of states $X_{s_i}$ within $\ell - 1$ distance of reaching $s_i$. The set $X_{s_i}$ represents the states that converge within $\ell$ steps if we add a point option from $s_i$ to $g$.
    \item Let $\mc{X}$ be a set of $X_{s_i}$ for every $s_i \in \mc{S} \setminus X^+_{g}$, where $X^+_{g}$ is a set of states that converges within $\ell$ without any options (thus can be ignored).  
    \item Solve the set-cover optimization problem to find a set of subsets that covers the entire state set using the approximation algorithm by \namecite{chvatal1979greedy}. This process corresponds to finding a minimum set of subsets $\{X_{s_i}\}$ that makes every state in $\mc{S}$ converge within $\ell$ steps.
    \item Generate a set of point options with initiation states set to one of the center states in the solution of the set-cover, and termination states set to the goal. 
\end{enumerate}
% Let $s_{x_j}$ be a set of states which converge to $\epsilon$-optimal value within $\ell - 1$ steps after $x_j$ reaches to a $\epsilon$-optimal value.
% If there is a point option from $x_j$ to the goal state, every state in $s_{x_j}$ converges within $\ell$ steps.
% Thus, our goal is to pick a smallest set of states $C$ so that every state in the MDP is covered: $\bigcup_{x \in \mc{C}} s_x = \mc{S}$.
% This problem can be considered as an optimization version of the set cover problem.

% --- Solve MOMI ---
%\input{algorithms/solve_momi.tex}
% ------------------

We compute a distance function $d_{\epsilon}: \mc{S} \times \mc{S} \rightarrow \mathbb{N}$\footnote{Formally, $d$ satisfies triangle inequality, but does not satisfy symmetry and indiscernibles.}, defined as follows:

\ddef{Distance $d_\epsilon(s_i, s_j)$}{
    $d_\epsilon(s_i, s_j)$ is the smallest number 
    $b$ such that for all $b' \geq b$, $V_b'(s_i)$ is $\epsilon$-optimal if we add a point option from $s_j$ to $g$, minus one.
}

More formally, let $d'_\epsilon(s_i)$ denote the number of iterations needed for the value of state $s_i$ to satisfy $|V(s_i) - V^*(s_i)| < \epsilon$, and let $d'_\epsilon(s_i, s_j)$ be an upper bound of the number of iterations needed for the value of $s_i$ to satisfy $|V(s_i) - V^*(s_i)| < \epsilon$, if the value of $s_j$ is initialized such that $|V(s_j) - V^*(s_j)| < \epsilon$. We define $d_\epsilon(s_i, s_j) := \min(d'_\epsilon(s_i) - 1, d'_\epsilon(s_i, s_j))$.
For simplicity, we use $d$ to denote the function $d_\epsilon$. 
Consider the following example. % \gdknote{I'm not sure if this level of detail is necessary.} \ynote{I believe we need this to make the algorithm reproducible.}

% --- MIMO Example ---
\begin{figure}[htb]
    \centering
    \begin{tikzpicture}
        \node [draw, circle] (x1) at (0, 4) {$s_1$};
        \node [draw, circle] (x2) at (0, 3) {$s_2$};
        \node [draw, circle] (x3) at (2, 4) {$s_3$};
        \node [draw, circle] (x4) at (2, 3) {$s_4$};
        \node [draw, circle] (x5) at (1, 2) {$s_5$};
        \node [draw, circle] (x6) at (1, 1) {$s_6$};
        
        \node [draw, circle, minimum size=0.8cm] (g) at (3, 1) {$g$};
        \node [draw, circle, minimum size=0.5cm] at (3, 1) {};

        \draw[->] (x1) -- (x2);
        \draw[->] (x3) -- (x4);
        \draw[->] (x2) -- (x5);
        \draw[->] (x4) -- (x5);
        \draw[->] (x5) -- (x6);
        \draw[->] (x6) -- (g);
        
        \draw[->, dashed] (x2.east) to [out=0,in=120] (g);
        \draw[->, dashed] (x4.east) to [out=0,in=90] (g.north);
    \end{tikzpicture}
    \caption{Example. Options discovered by \mimoalg{} with $k=2$ are denoted by the dashed lines.}
    \label{fig:mimo-example}
\end{figure}

\begin{table}[htb]
    \centering
    \begin{tabular}{c|cccccc}
        $s \setminus s'$   & $s_1$ & $s_2$ & $s_3$ & $s_4$ & $s_5$ & $s_6$ \\ \hline
        $s_1$ & 0 & 1 & 3 & 3 & 2 & 3 \\
        $s_2$ & 2 & 0 & 2 & 2 & 1 & 2 \\
        $s_3$ & 3 & 3 & 0 & 1 & 2 & 3 \\
        $s_4$ & 2 & 2 & 2 & 0 & 1 & 2 \\
        $s_5$ & 1 & 1 & 1 & 1 & 0 & 1 \\
        $s_6$ & 0 & 0 & 0 & 0 & 0 & 0 \\
    \end{tabular}
    \caption{$d_0(s, s')$ for Figure \ref{fig:mimo-example}.}
    \label{tab:d}
\end{table}

\begin{example}
Table \ref{tab:d} is a distance function for the MDP shown in Figure \ref{fig:mimo-example}. For a deterministic MDP, $d_0(s)$ corresponds to the number of edge traversals from state $s$ to $g$, where we have edges only for those that corresponds to the state transition by the optimal actions. The quantity $d_0(s, s') - 1$ is the minimum of $d_0(s)$ and one plus the number of edge traversals from $s$ to $s'$.
\end{example}

Note that we only need to solve the MDP once to compute $d$. $d(s, s')$ can be computed once you solved the MDP without any options and store all value functions $V_i$ ($i=1,...b$) until convergence as a function of $V_1$: $V_i(s) = f(V_1(s_0), V_1(s_1),...)$. If we add a point option from $s'$ to $g$, then $V_1(s') = V^*(s')$.
Thus, $d(s, s')$ is the smallest $i$ where $V_{i}(s)$ reaches $\epsilon$-optimal if we replace $V_1(s')$  with $V^*(s')$ when computing $V_i(s)$ as a function of $V_1$. % \ynote{this is as hard as solving the MDP |S| times...}

\begin{example}
We use the MDP shown in Figure \ref{fig:mimo-example} as an example. Consider the problem of finding a set of options so that the MDP can be solved within $2$ iterations. We generate an instance of a set-cover optimization problem.
The set of elements for the set cover is the set of states of the MDP that do not reach their optimal value within $\ell$ steps without any options $\mc{S} \setminus X^{+}_{g}$. Here, we denote a set of nodes that can be solved within $\ell$ steps by $X^{+}_{g}$. In the example, $\mc{U} = S \setminus X^{+}_{g} = \{s_1, s_2, s_3, s_4\}$.
A state $s$ is included in a subset $X_{s'}$ iff $d(s, s') \leq \ell - 1$. For example, $X_{s_1} = \{s_1\}, X_{s_2} = \{s_1, s_2\}$.
Thus, the set of subsets are given as:
    $X_{s_1} = \{s_1\},
    X_{s_2} = \{s_1, s_2\},
    X_{s_3} = \{s_3\},
    X_{s_4} = \{s_3, s_4\}$.
In this case, the approximation algorithm finds the optimal solution $\mc{C} = \{X_{s_2}, X_{s_4}\}$ for the set-cover optimization problem ($\mc{U}, \mc{X}$).
We generate a point option for each state in $\mc{C}$. Thus, the output of the algorithm is a set of two point options from $s_2$ and $s_4$ to $g$.
%\gdknote{We should find the actual solution, not hypothesize one.}
\end{example}

\begin{theorem}
    \momialg{} has the following properties:
    \begin{enumerate}
        \itemsep0em 
        \item \momialg{} runs in polynomial time.
        \item It guarantees that the MDP is solved within $\ell$ iterations using the option set acquired by \momialg{} $\mc{O}$. 
        \item If the MDP is deterministic, the option set is at most $O(\log n)$ times larger than the smallest option set possible to solve the MDP within $\ell$ iterations.
        % \item If the MDP is deterministic, the option set is at most $\max_{s \in \mc{S}} X_s$ times larger than the smallest option set possible to solve the MDP within $\ell$ iterations.
    \end{enumerate}
\end{theorem}
\begin{proof}
    See the supplementary material.
\end{proof}

Note that the approximation bound for a deterministic MDP will inherent any improvements to the approximation algorithm for set cover. Set cover is known to be $\NP$-hard to approximate up to a factor of $(1 - o(1)) \log n$ \cite{dinur2014analytical}, thus there may be an improvement on the approximation ratio for the set cover problem, which will also improve the approximation ratio of \momialg.

% \dnote{Should we cite these folks:~\cite{dinur2014analytical}? They have the best known lower bound for approximating set cover (as far as I can tell).} \ynote{we can cite but the question is how this work relates to our work. They are not providing an algorithm for set cover optimization. I guess we can say that if better approximation algorithm is provided for SetCover, then our algorithm is improved. We can hope that because inapproximability is only proven up to (1 - o(1)) ln n, there might be an improvement on suboptimality in SetCover.} \dnote{Yes, what I had in mind was the latter -- George mentioned it might be worthwhile to state how efficient our algorithm is compared to the known lower bound (but as you said, it's important to note there's no algorithm that achieves it so far). Not crucial, just if we have space}

% -- MIMO --
\subsection{\mimoalg{}}
\label{sec:alg-mimo}

The outline of the approximation algorithm for MIMO (\mimoalg{}) is as follows.
\begin{enumerate}
    \item Compute $d_\epsilon(s, s'): S \times S \rightarrow \mathbb{N}$ for each pair of states. %\gdknote{Wait, $d$ maps $s$, and $s'$ a subgoal state, to iterations? This seems wildly expensive to compute.} \ynote{Yes, it is expensive.}
    
    \item Using this distance function, solve an asymmetric $k$-center problem, which finds a set of center states that minimizes the maximum number of iterations for every state to converge.
    
    \item Generate point options with initiation states set to the center states in the solution of the asymmetric $k$-center, and termination states set to the goal.
\end{enumerate}

% MIMO Pseudocode.
%\input{algorithms/solve_mimo.tex}
%Now in more detail.

As in \momialg{}, we first compute the distance function. Then, we exploit this characteristic of $d$ and solve the asymmetric $k$-center problem~\cite{panigrahy1998ano} on $(\mc{U}, d, k)$ to get a set of centers, which we use as initiation states for point options.
The asymmetric $k$-center problem is a generalization of the metric $k$-center problem where the function $d$ obeys the triangle inequality, but is not necessarily symmetric:

% --- Asymmetric K Center Computational Problem ---
\ddef{AsymKCenter}{ \\
    {\bf Given} a set of elements $\mc{U}$, a function $d: \mc{U} \times \mc{U} \rightarrow \mathbb{N}$, and an integer $k$, {\bf return} $\mc{C}$ that minimizes $P(\mc{C}) = \max_{s \in U} \min_{c \in \mc{C}} d(s, c)$ subject to $|\mc{C}| \leq k$.
}

We solve the problem using a polynomial-time approximation algorithm proposed by \namecite{archer2001two}. The algorithm has a suboptimality bound of $O(\log^*k)$\footnote{$\log^*$ is the number of times the logarithm function must be iteratively applied before the result is less than or equal to 1.} where $k < |\mc{U}|$. It is known that the problem cannot be solved within a factor of $\log^* |\mc{U}| - \theta(1)$ unless $P=\NP$~\cite{chuzhoy2005asymmetric}. %Since the problem is trivial if $k \geq |\mc{U}|$, $k$ is usually smaller than $|\mc{U}|$. Thus the approximation is known to be the tightest in terms of input size. % \gdknote{How does n relate to k in the previous sentence?}
As the procedure by \namecite{archer2001two} often finds a set of options smaller than $k$, we generate the rest of the options by greedily adding $\log k$ options at once. See the supplementary material for details.
Finally, we generate a set of point options with initiation-states set to one of the centers and the termination state set to the goal state of the MDP.
That is, for every $c$ in $\mc{C}$, we generate a point option starting from $c$ to the goal state $g$.

\begin{example}
Consider an MDP shown in Figure \ref{fig:mimo-example}. The distance $d_0$ for the MDP is shown in Table \ref{tab:d}. Note that $d(s, s') \leq d(s, g)$ holds for every $s, s'$ pair.
Let us first consider finding one option ($k=1$). This process corresponds to finding a column with the smallest maximum value in the Table~\ref{tab:d}.
The optimal point option is from $s_5$ to $g$ as it has the smallest maximum value in the column.
If $k=2$, an optimal set of options is from $s_2$ and $s_4$ to $g$. Note that the optimal option for $k=1$ is not in the optimal option set of size 2. This example shows that the  strategy of greedily adding options does not find the optimal set. In fact, the improvement $L_{\epsilon, V_0}(\emptyset) - L_{\epsilon, V_0}({\mc{O}})$ on by the greedy algorithm can be arbitrary small (i.e. 0) compared to the optimal option (see Proposition 1 in the supplementary material for a proof).
\end{example}

\begin{theorem}
    \mimoalg{} has the following properties:
    \begin{enumerate}
        \itemsep0em 
        \item \mimoalg{} runs in polynomial time.
        \item If the MDP is deterministic, it has a bounded suboptimality of $O(\log^* k)$.
        \item The number of iterations to solve the MDP using the option set acquired is upper bounded by $P(\mc{C})$. % I don't see why P(C)---the value of the the AsymKCenter solution---need not be arbitrarily bad. I suppose this is what experiments are for? % Yes, P(C) can be arbitrary bad and we have an example for that.
    \end{enumerate}
\end{theorem}
\begin{proof}
    See the supplementary material.
\end{proof}

%%%%%%%%%%%%%%%%%%%%%%%%%%%%%%%%%%%%%%%%%%%%%%%%%%%%%%%%%%%%%%%%
\section{Experiments}
\label{sec:experiments}

\begin{figure*}
    \centering
    \newcommand{\figsize}{0.14}
    \subfloat[optimal $k = 2$]{\includegraphics[width=\figsize\textwidth]{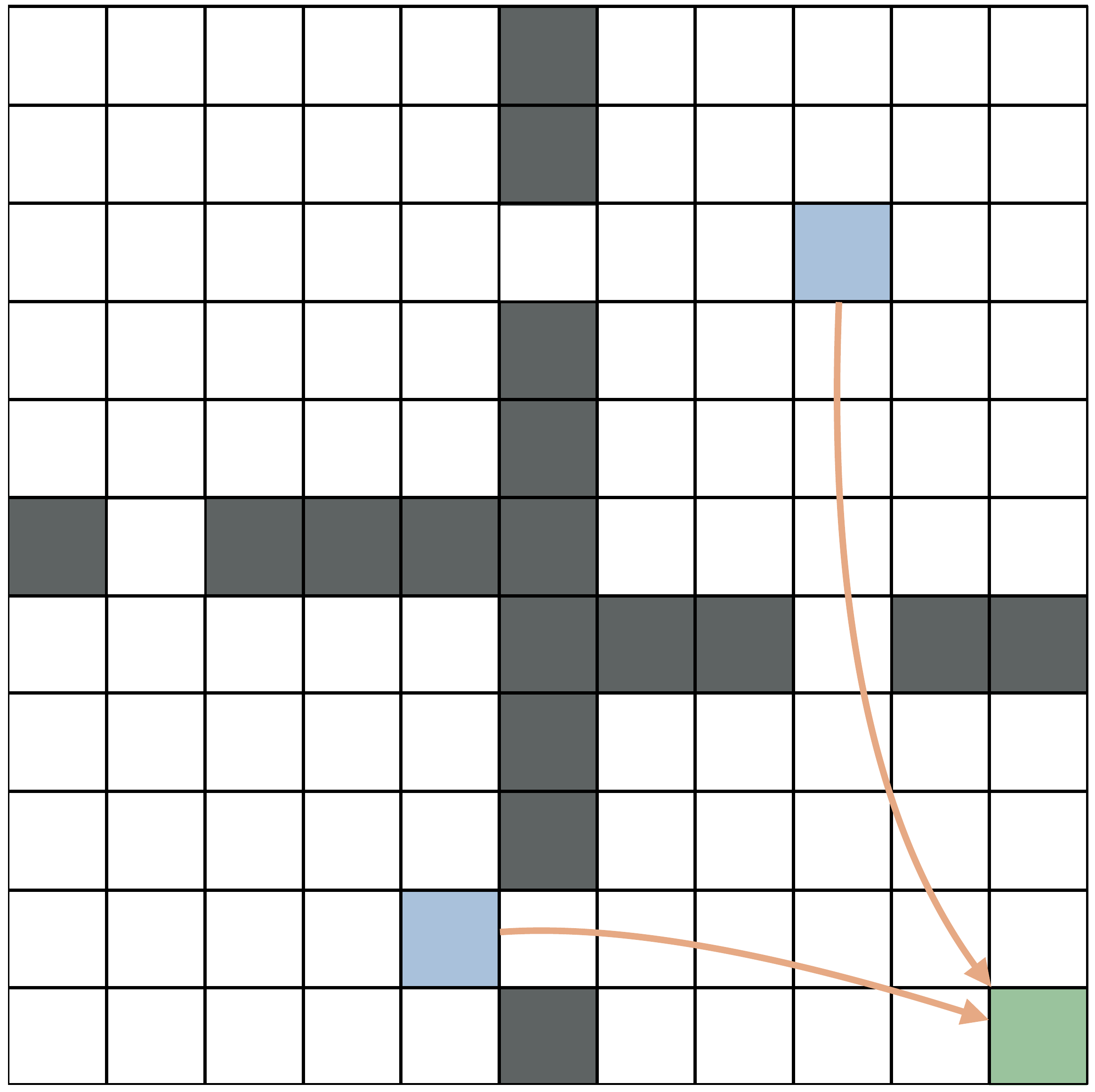} \label{mimo2op}} \hspace{1mm}
    \subfloat[approx. $k = 2$]{\includegraphics[width=\figsize2\textwidth]{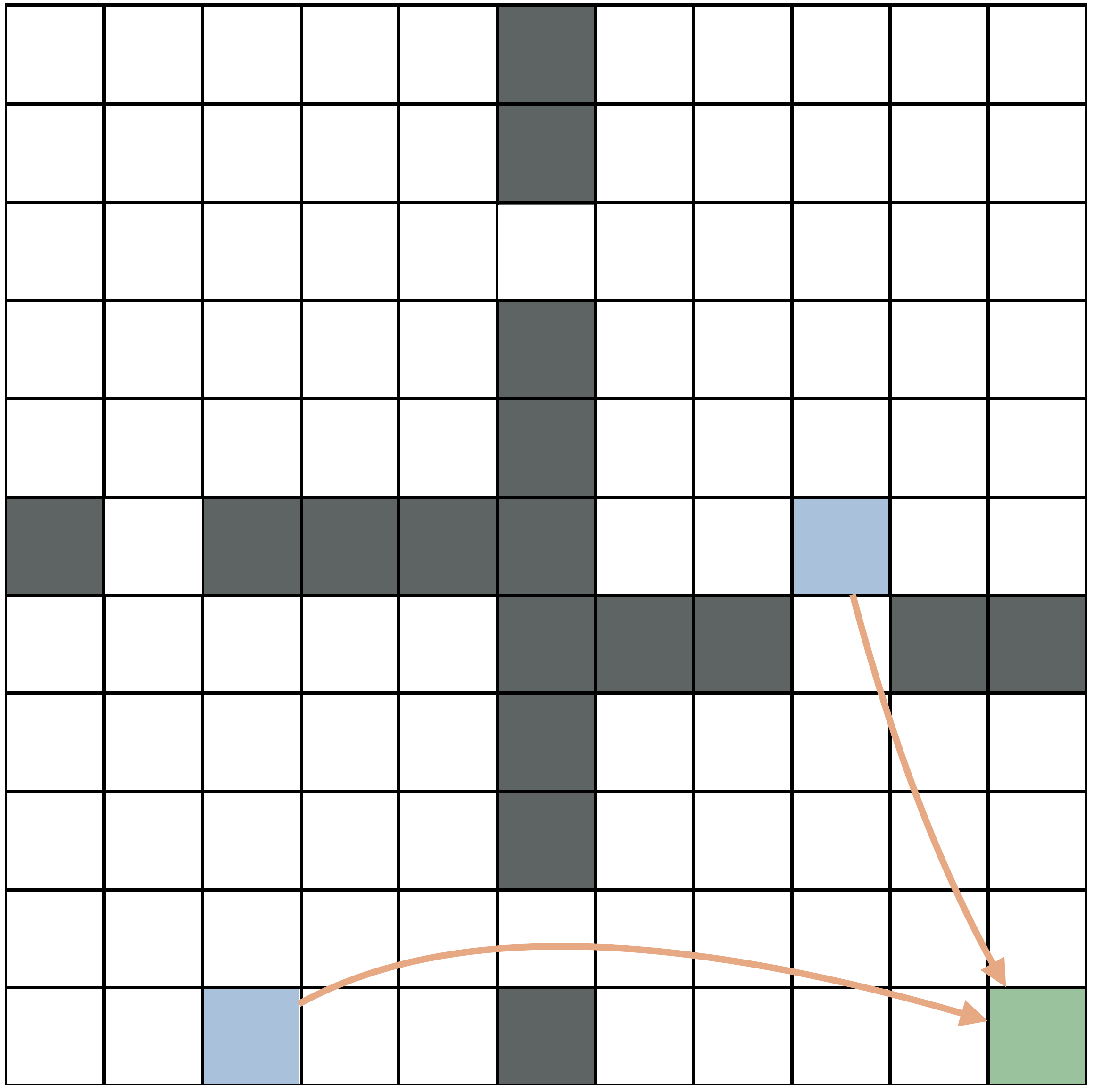} \label{mimo2ap}} \hspace{1mm}
    \subfloat[optimal $k = 4$]{\includegraphics[width=\figsize\textwidth]{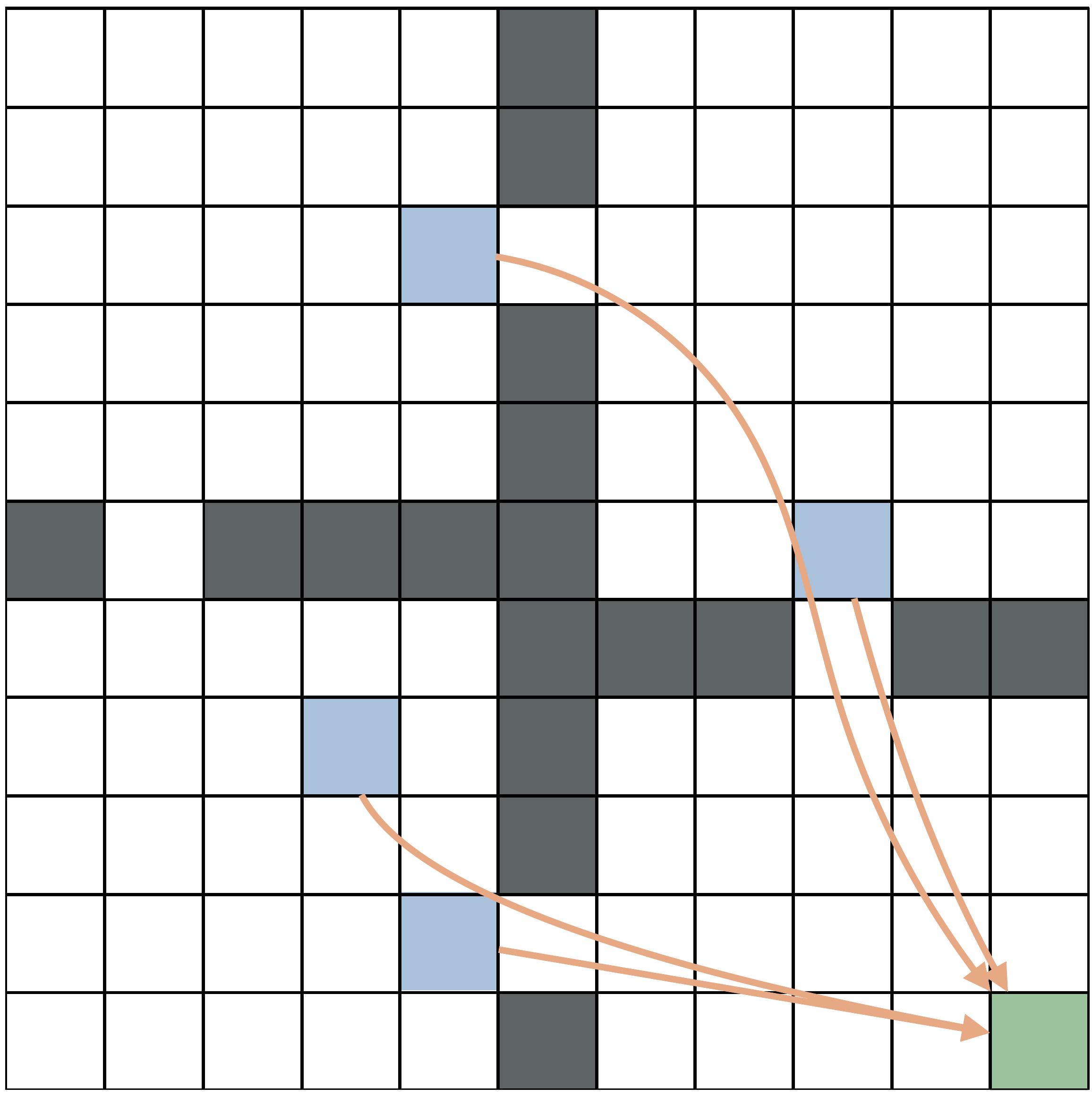} \label{mimo4op}} \hspace{1mm}
    \subfloat[approx. $k = 4$]{\includegraphics[width=\figsize\textwidth]{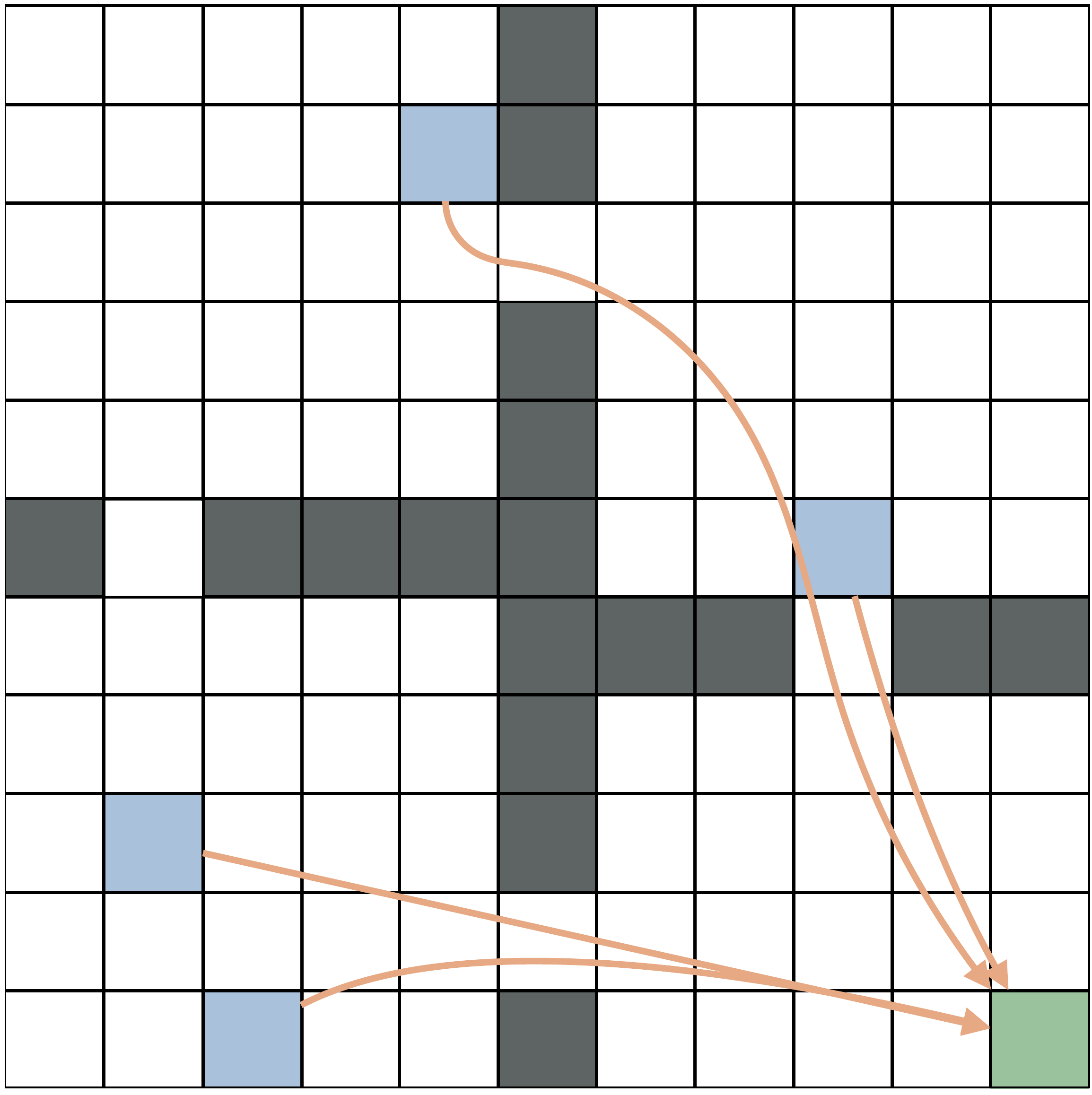} \label{mimo4ap}} \hspace{1mm}
    \subfloat[Betweenness]{\includegraphics[width=\figsize\textwidth]{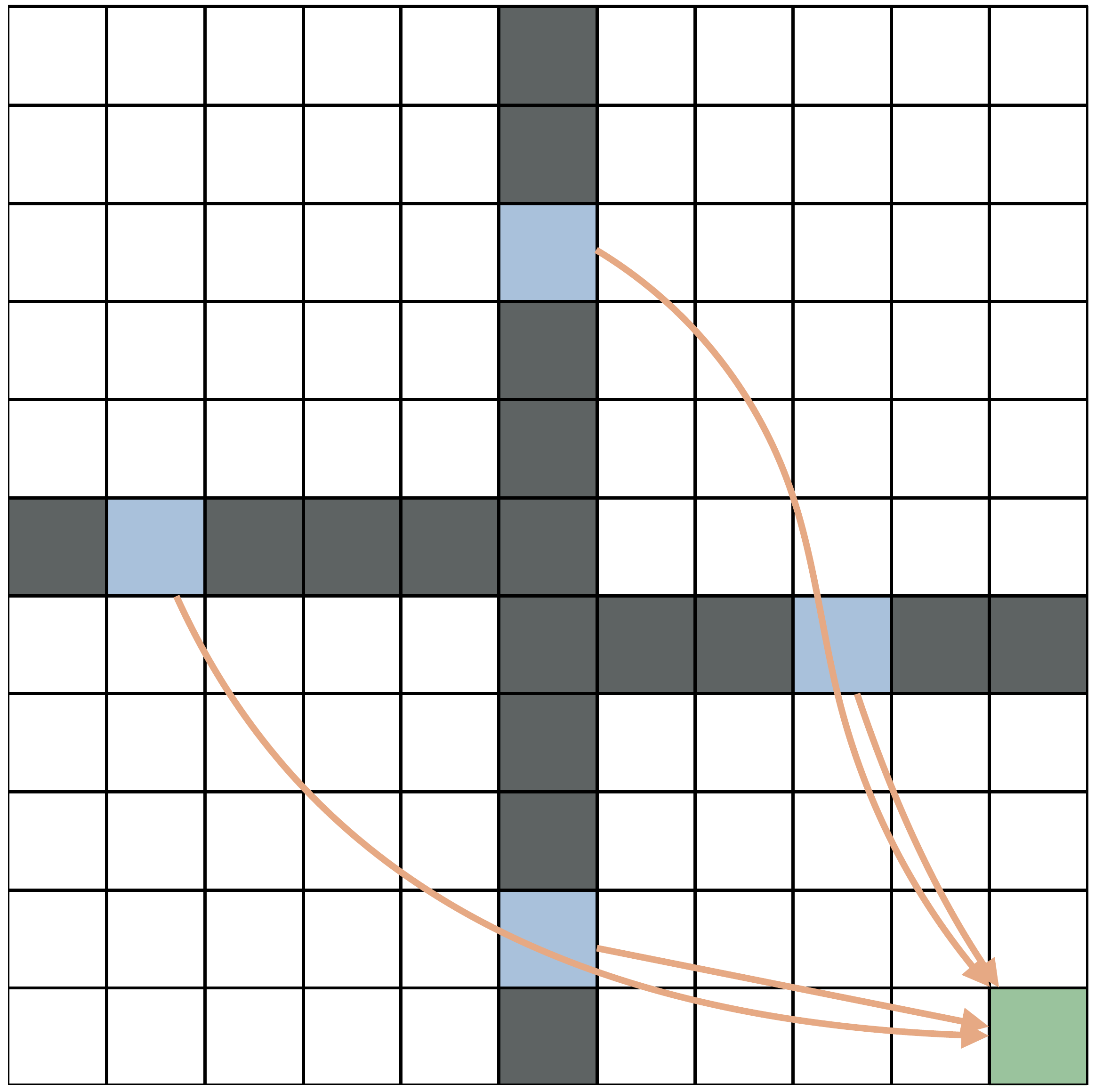} \label{mimo4bet}} \hspace{1mm}
    \subfloat[Eigenoptions]{\includegraphics[width=\figsize\textwidth]{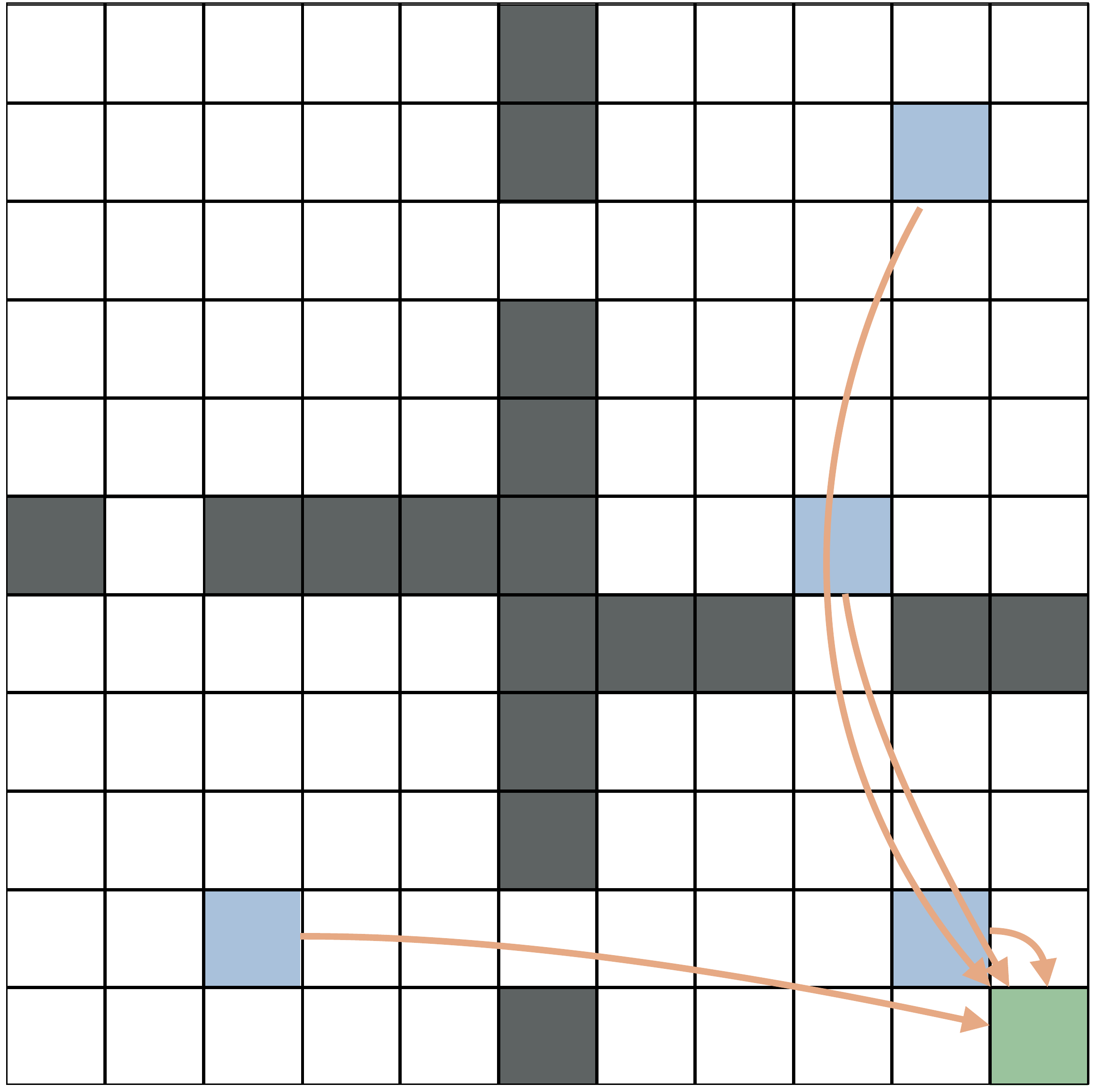} \label{mimo4eg}} 

    \caption{Comparison of the optimal point options with options generated by the approximation algorithm \mimoalg{}. The green square represents the termination state and the blue squares the initiation states.
    Observe that the approximation algorithm is similar to that of optimal options. Note that the optimal option set is not unique: there can be multiple optimal option sets, and we are visualizing just one returned by the solver.}
    \label{fig:fourroom-viz}
\end{figure*}

\begin{figure*}
    \centering
    \newcommand{\figsizeq}{0.24}
    \subfloat[Four Room (MIMO)]{\includegraphics[width=\figsizeq\textwidth]{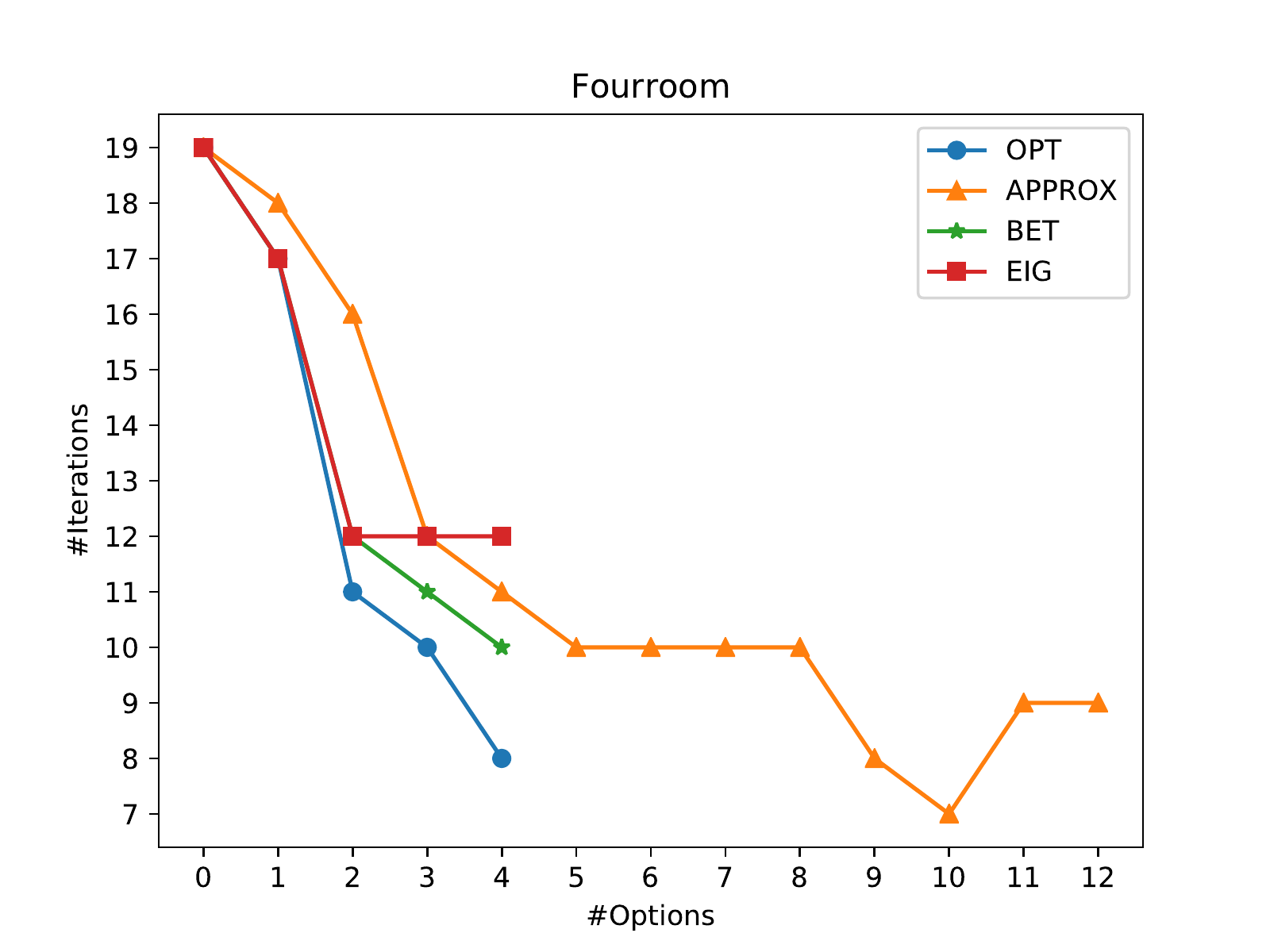} \label{fig:mimofour}}
    \subfloat[$9\times 9$ grid (MIMO)]{\includegraphics[width=\figsizeq\textwidth]{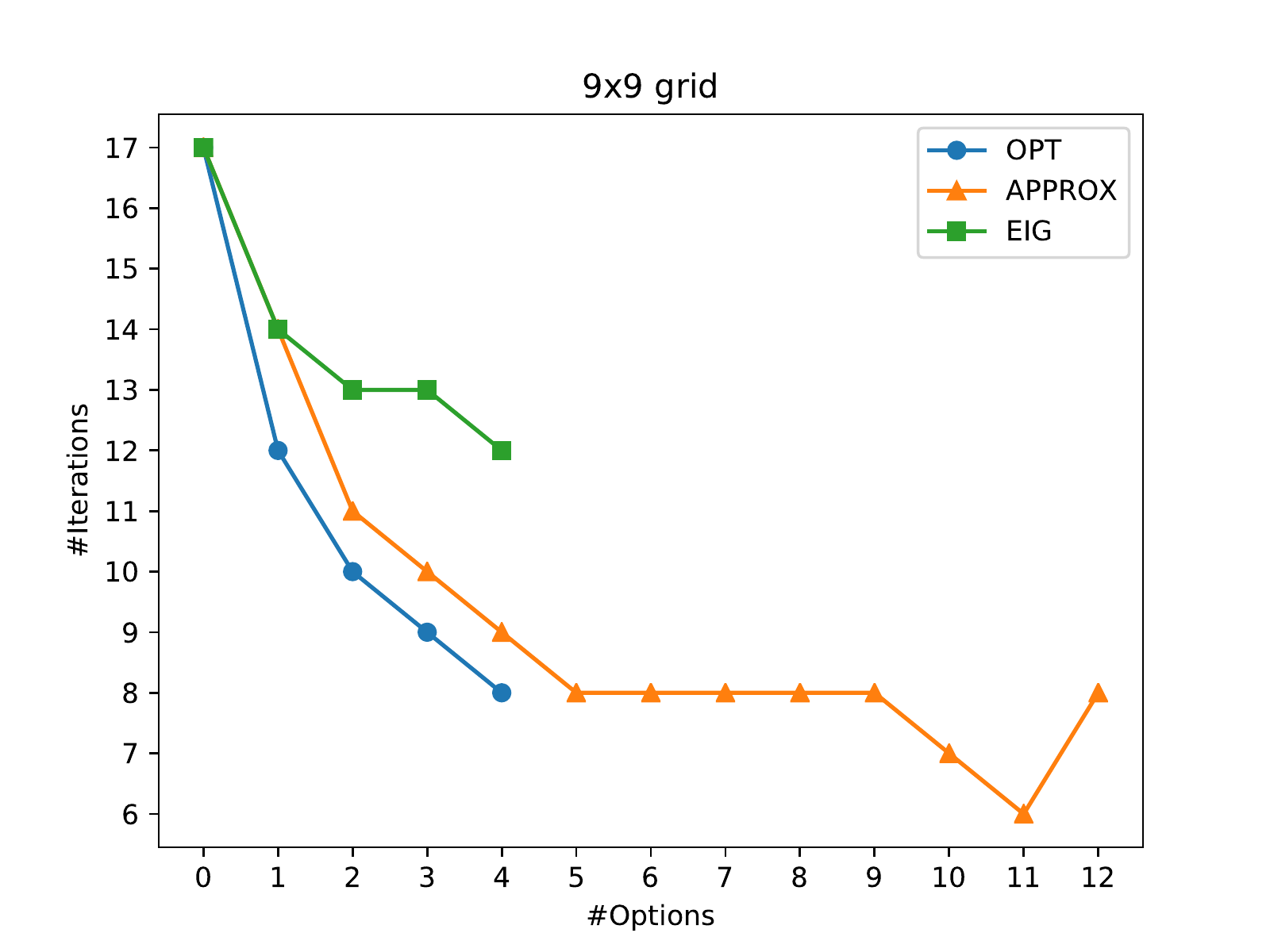} \label{fig:mimogrid}}
    %\subfloat[Imaze]{\includegraphics[width=0.3\textwidth]{figures/mimo/Imaze.pdf} \label{mimomaze}}
    \subfloat[Four Room (MOMI)]{\includegraphics[width=\figsizeq\textwidth]{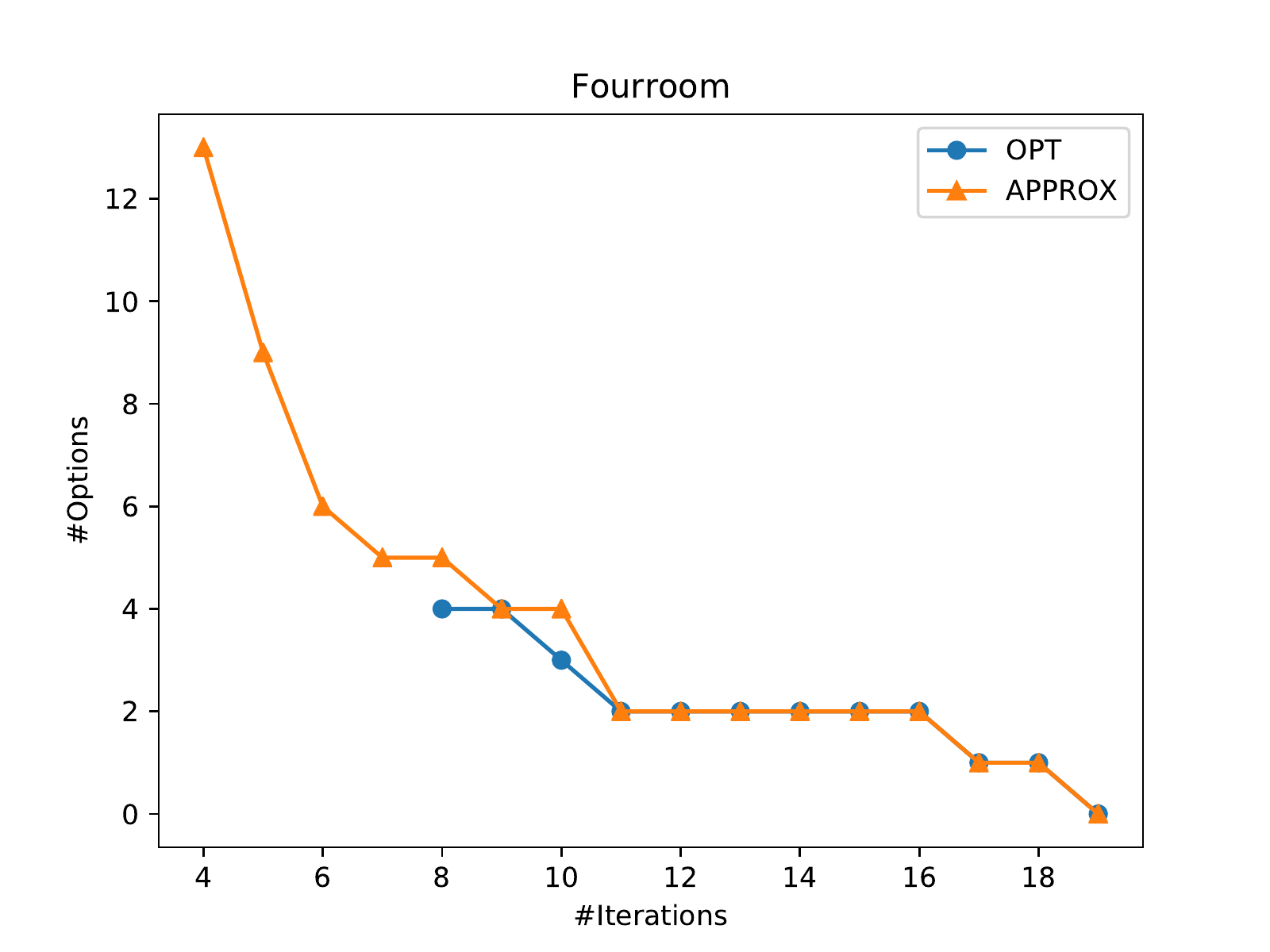} \label{fig:momifour}}
    \subfloat[$9\times 9$ grid (MOMI)]{\includegraphics[width=\figsizeq\textwidth]{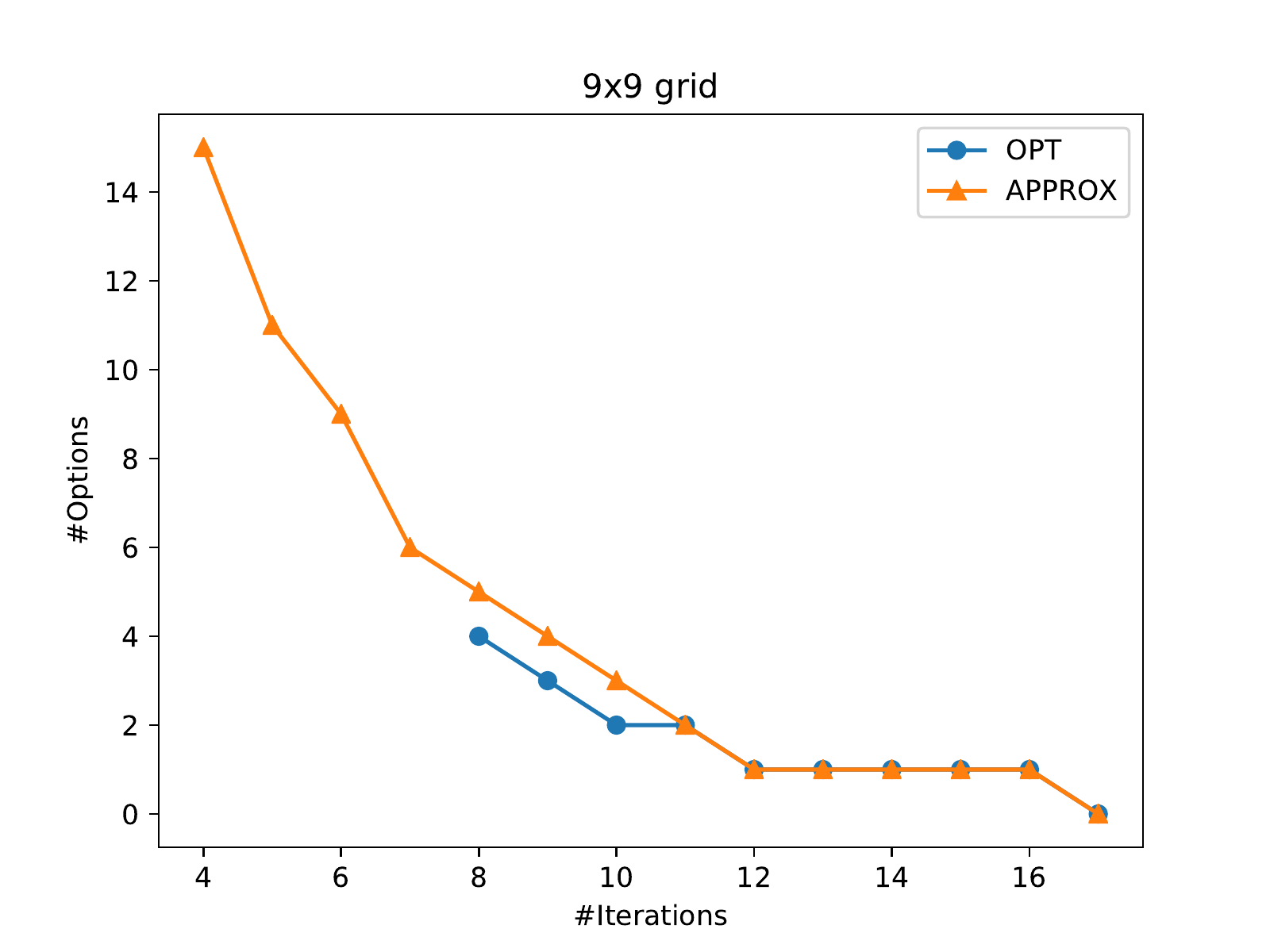} \label{fig:momigrid}}
        %\subfloat[Imaze]{\includegraphics[width=0.3\textwidth]{figures/momi/Imaze.pdf}}
    \caption{MIMO and MOMI evaluations. Parts~(a)--(b) show the number of iterations for VI using options generated by \mimoalg{}. Parts~(c)--(d) show the number of options generated by \momialg{} to ensure the MDP is solved within a given number of iterations. OPT: optimal set of options. APPROX: a bounded suboptimal set of options generated by \mimoalg{} an \momialg{}. BET: betweenness options. EIG: eigenoptions.}
    \label{fig:numiter-op}
\end{figure*}

We evaluate the performance of the value-iteration algorithm using options generated by the approximation algorithms on several grid-based simple domains.
% Our code is freely available online.\footnote{\it Redacted for review} \ynote{Refactor the code and release it}

We ran the experiments on an $11 \times 11$ four-room domain and a $9 \times 9$ grid world with no walls.
In both domains, the agent's goal is to reach a specific square. The agent can move in the usual four directions  but cannot cross walls.

\noindent{\bf Visualizations}: First, we visualize a variety of option types, including the optimal point options, those found by our approximation algorithms, and several option types proposed in the literature.
We computed the optimal set of point options by enumerating every possible set of point options and picking the best. As an optimal set of options is not unique, we picked one arbitrary.
We are only able to find optimal solutions up to four options within 10 minutes, while the approximation algorithm could find any number of options within a few minutes.
For eigenoptions, we ignored the eigenvector corresponding to the smallest eigenvalue (=0) in the graph Laplacian because it has a constant value for every state.
Both betweenness options and eigenoptions are discovered by a polynomial time algorithm, thus able to discover within a few minutes. % betweenness can be computed wihtin |V||E| time.
%Approximation algorithms, betweenness options and eigenoptions are computed within polynomial time.
Figure~\ref{fig:fourroom-viz} shows the optimal and bounded suboptimal set of options computed by \mimoalg{}. See the supplementary material for visualizations for the $9\times 9$ grid domain. 

Figure \ref{mimo4bet} shows the four bottleneck states with highest shortest-path betweenness centrality in the state-transition graph~\cite{csimcsek2009skill}.
Interestingly, the optimal options are quite close to the bottleneck states in the four-room domain, suggesting that bottleneck states are also useful for planning as a heuristic to find important subgoals. %We generate a set of point options from each bottleneck states to the goal state. % if such states are available. \gdknote{The previous sentence did not parse.}

Figure~\ref{mimo4eg} shows the set of subgoals discovered by graph Laplacian analysis
% . We follow the discovery of subgoals
following
the method of \namecite{machado2018laplacian}. While they proposed to generate options to travel between subgoals for reinforcement learning, we generate a set of point options from each subgoal to the goal state as that is a better use of the subgoals for planning setting.

% \dnote{If space, I'd like to add a bit more analysis of the visuals}

% Note that bottleneck options are generated to speedup reinforcement learning, not planning. % Also, bottleneck options and eigenoptions are generated without using the information of the position of the goal state - they are generated only from state transitions.
% We ran value iteration using a set of point options from each bottleneck states to the goal state.

\noindent{\bf Quantitative Evaluation}: Next, we run value iteration using the set of options generated by \mimoalg{} and \momialg{}.
Figures~\ref{fig:mimofour} and~\ref{fig:mimogrid} show the number of iterations on the four-room and the $9\times 9$ grids using a set of options of size $k$.
The experimental results suggest that the suboptimal algorithm finds set of options similar to, but not quite as good as, the optimal ones.
For betweenness options and eigenoptions, we evaluated every subset of options among the four and present results for the best subset found.
Because betweenness options are placed close to the optimal options, the performance is close to optimal especially when the number of options is small.

In addition, we used \momialg{} to find a minimum option set to solve the MDP within the given number of iterations.
Figures~\ref{fig:momifour} and~\ref{fig:momigrid} show the number of options generated by \momialg{} compared to the minimum number of options.
% To the best of our knowledge, there is no algorithm proposed to find a set of options that guarantees that the MDP can be solved within a given number of iterations.\gdknote{I did not follow the logic of this last statement.}

% ------------------
% -- Related Work --
% ------------------
\section{Related Work}
\label{sec:relatedwork}

Many heuristic algorithms have proposed to discover options useful for some purposes \cite{iba1989heuristic,mcgovern2001automatic,menache2002q,stolle2002learning,Simsek04,csimcsek2009skill,konidaris2009skill,machado2018laplacian,eysenbach2018diversity}.
These  algorithms seek to capture varying intuitions about what makes behavioral abstraction useful.
\citeauthor{jong2008utility} (\citeyear{jong2008utility}) sought to investigate the utility of options empirically and pointed out that introducing options might worsen learning performance. They argued that options can potentially improve the learning performance by encouraging exploitation or exploration. %\gdknote{Didn't Jong also argue that the wrong set of options can result in bad performance, for learning?} \ynote{yes}
For example, some works investigate the use of bottleneck states \cite{stolle2002learning,csimcsek2009skill,menache2002q,lehnert2018value}.
\namecite{stolle2002learning} proposed to set states with high visitation count as subgoal states, resulting in identifying bottleneck states in the four-room domain.
\namecite{csimcsek2009skill} generalized the concept of the bottleneck to (shortest-path) betweenness of the graph to capture how pivotal the state is.
\namecite{menache2002q} used a learned model of the environment to run a Max-Flow/Min-Cut algorithm to the state-space graph to identify bottleneck states.
These methods generate options to leverage the idea that subgoals are states visited most frequently.
On the other hand, \namecite{Simsek04} proposed to generate options to encourage exploration by generating options to relatively novel states, encouraging exploration. \namecite{eysenbach2018diversity} instead proposed learning a policy for each option so that the diversity of the trajectories by the set of options are maximized.
These methods generate options to explore infrequently visited states.
\namecite{harb2017waiting} proposed to formulate good options to be options which minimize the deliberation costs in the bounded rationality framework \cite{simon1957models}.
% The problem of discovering efficient behavioral abstraction in reinforcement learning is still an open question.
The problem of finding minimum state-abstraction with bounded performance loss was studied by \namecite{even2003approximate}. They showed it is $\NP$-hard and proposed a polynomial time bicriteria approximation algorithm.

For planning, several works have shown empirically that adding a particular set of options or macro-operators can speed up planning algorithms~\cite{francis1993utility,sutton1998reinforcement,silver2012compositional,konidaris2016constructing}.
% In terms of theoretical analysis, 
\namecite{mann2015approximate} analyzed the convergence rate of approximate value iteration with and without options and showed that options lead to faster convergence if their duration are longer and a value function is initialized pessimistically.
As in reinforcement learning, how to find efficient temporal abstractions for planning automatically remains an open question.
%%%%%%%%%%%%%%%%%%%%%%%%%%%%%%%%%%%%%%%%%%%%%%%%%%%%%%%%%%%%%%%%

% -----------------
% -- Conclusions --
% -----------------
\section{Conclusions}
\label{sec:conclusions}

We considered a fundamental theoretical question concerning the use
of behavioral abstractions to solve  MDPs.
We considered two problem formulations for finding options: (1) minimize the size of option set given a maximum number of iterations (MOMI) and (2) minimize the number of iterations given a maximum size of option set (MIMO).
We showed that the two problems are both computationally intractable, even for deterministic MDPs.
For each problem, we produced a polynomial-time algorithm for MDPs with bounded reward and goal states, with bounded optimality for deterministic MDPs. Although these algorithms are not practical for a single-task planning, we believe these algorithms may be  a useful foundation for future option discovery methods.
% There are several factors we did not consider in this paper. First, we ignore the cost of acquiring the model of the option. Second, we focused on Value iteration algorithm. The 
%\dnote{I would pitch future work, too: the learning/model-based RL variant. Something like:}  \ynote{I like it.}
In the future, we are interested in using the insights established here to develop principled option-discovery algorithms for model-based reinforcement learning. Since we now know which options minimize planning time, we can better guide model-based agents toward learning them and potentially reduce sample complexity considerably.

\section*{Acknowledgments}

We would like to thank the anonymous reviewer for their advice and suggestions to improve the inapproximability result for MOMI. We would like to thank Maehara Takanori and Kazuki Yoshizoe for their advice on formulating the problems.

% --- Bibliography ---
\bibliographystyle{icml2019}
\bibliography{ms}

%%%%%%%%%%%%%%%%%%%%%%%%

\clearpage

\input{appendix.tex}

\end{document}

%% file: commands.tex
\newcommand{\bigmid}{\mathrel{\Big|}}

\newcommand{\durl}[1]{\textcolor{blue}{\underline{\url{#1}}}}

\newcommand{\eps}{\varepsilon}

\newcommand{\mc}[1]{\mathcal{#1}}

\newcommand{\bE}{\mathbb{E}}

\newcommand{\ra}{\rightarrow}

\DeclareMathOperator*{\argmax}{arg\,max}

\newtheorem{theorem}{Theorem}

\newtheorem{lemma}{Lemma}
\newtheorem{proposition}{Proposition}
\newtheorem{theorems}{Theorem}[theorem]

\definecolor{dblue}{RGB}{98, 140, 190}
\definecolor{dgreen}{RGB}{113, 198, 113}
\definecolor{dpink}{RGB}{207, 166, 208}
\definecolor{dgold}{RGB}{197, 193, 170}

\newcounter{DaveDefCounter}
\newcommand{\ddef}[2]
{
\begin{mdframed}[roundcorner=1pt, backgroundcolor=white]
\vspace{1mm}
{\bf Definition \theDaveDefCounter} (#1): {\it #2}
\stepcounter{DaveDefCounter}
\end{mdframed}
}

%% file: proofs/mimo_momi_np_hard.tex
\begin{proof}
    % TODO: reduction
    We consider a problem OI-DEC which is a decision version of MOMI and MIMO. 
    The problem asks if we can solve the MDP within $\ell$ iterations using at most $k$ point options.
    
    \ddef{OI-DEC}{ \\
        {\bf Given} an MDP $M$, a non-negative real-value $\epsilon$, an initial value function $V_0$, and integers $k$ and $\ell$, {\bf return} `Yes' if the there exists an option set $\mc{O}$ such that $\mc{O} \subseteq \mc{O}_p$, $|\mc{O}| \leq k$ and $L(\mc{O}) \leq \ell$. `No' otherwise.
    }
    
    We prove the theorem by reduction from the decision version of the set-cover problem---known to be NP-complete---to OI-DEC.
    The set-cover problem is defined as follows.

    \ddef{SetCover-DEC} { \\
        {\bf Given} a set of elements $\mc{U}$, a set of subsets $\mc{X} = \{X \subseteq \mc{U}\}$, and an integer $k$, {\bf return} `Yes' if there exists a cover $\mc{C} \subseteq \mc{X}$ such that $\bigcup_{X \in \mc{C}} X = \mc{U}$ and $|\mc{C}| \leq k$. `No' otherwise.
    }
    
    If there is some $u \in \mc{U}$ that is not included in at least one of the subsets $X$, then the answer is `No'.
    Assuming otherwise, we construct an instance of a shortest path problem (a special case of an MDP problem) as follows (Figure \ref{fig:reduction}).
    There are four types of states in the MDP: (1) $u_i \in \mc{U}$ represents one of the elements in $\mc{U}$, (2) $X_i \in \mc{X}$ represents one of the subsets in $\mc{X}$, (3) $X'_i \in \mc{X}'$: we make a copy for every state $X_i \in \mc{X}$ and call them $X'_i$, (4) a goal state $g$.
    Thus, the state set is $\mc{U} \cup \mc{X} \cup \mc{X}' \cup \{g\}$.
    We build edges between states as follows:
    (1) $e(u, X) \in E$ iff $u \in X$: For $u \in \mc{U}$ and $X \in \mc{X}$, there is an edge between $u$ and $X$.
    (2) $\forall X_i \in \mc{X}$, $e(X_i, X'_i) \in E$: For every $X_i \in \mc{X}$, we have an edge from $X_i$ to $X'_i$. (3) $\forall e(X', g) \in E$: for every $X' \in \mc{X}'_i$ we have an edge from $X_i$ to the goal $g$.
    This construction can be done in polynomial time. % linear time?

    Let $M$ be the MDP constructed in this way.
    We show that SetCover($\mc{U}, \mc{X}, k$) = OI-DEC($M, V_0 = 0, k, 2$).
    Note that by construction every state $X_i$, $X'_i$, and $g$ converges to its optimal value within 2 iterations as it reaches the goal state $g$ within 2 steps. A state
    $u \in \mc{U}$ converges within 2 steps if and only if there exists a point option (a) from $X$ to $g$ where $u \in X$, (b) from $u$ to $X'$ where $u \in X$, or (c) from $u$ to $g$. For options of type (b) and (c), we can find an option of type (a) that makes $u$ converge within 2 steps by setting the initial state of the option to $\mc{I}_o = X$, where $u \in X$, and the termination state to $\beta_o = g$.
    % Then every $x \in X$ converges within 2 steps if we set options from $s$ to $g$ so that for every $x$ there exists $s$ such that $x \in s$.
    Let $\mc{O}$ be the solution of OI-DEC($M, k, 2$).
    If there exists an option of type (b) or (c), we can swap them with an option of type (a) and still maintain a solution.
    Let $\mc{C}$ be a set of initial states of each option in $\mc{O}$ ($\mc{C} = \{\mc{I}_o | o \in \mc{O}\}$).
    This construction exactly matches the solution of the SetCover-DEC.
    
    \begin{figure}
        \centering
        \begin{tikzpicture}
            \node [draw, circle] (x1) at (0, 3) {$u_1$};
            \node [draw, circle] (x2) at (1, 3) {$u_2$};
            \node [draw, circle] (x3) at (2, 3) {$u_3$};
            \node [draw, circle] (x4) at (3, 3) {$u_4$};
            \node [draw, circle] (x5) at (4, 3) {$u_5$};

            \node [draw, circle, label=center:$X_1$] (s1) at (1, 2) {\phantom{$u_1$}};
            \node [draw, circle, label=center:$X_2$] (s2) at (3, 2) {\phantom{$u_1$}};

            \node [draw, circle, label=center:$X'_1$] (sp1) at (1, 1) {\phantom{$u_1$}};
            \node [draw, circle, label=center:$X'_2$] (sp2) at (3, 1) {\phantom{$u_1$}};

            \node [draw, circle, minimum size=0.8cm] (g) at (2, 0) {$g$};
            \node [draw, circle, minimum size=0.5cm] at (2, 0) {};

            \draw[->] (x1) -- (s1);
            \draw[->] (x2) -- (s1);
            \draw[->] (x3) -- (s1);
            \draw[->] (x3) -- (s2);
            \draw[->] (x4) -- (s2);
            \draw[->] (x5) -- (s2);
            \draw[->] (s1) -- (sp1);
            \draw[->] (s2) -- (sp2);
            \draw[->] (sp1) -- (g);
            \draw[->] (sp2) -- (g);
        \end{tikzpicture}
        \caption{Reduction from SetCover-DEC to OI-DEC. The example shows the reduction from an instance of SetCover-DEC which asks if we can pick two subsets from $\mc{X} = \{X_1, X_2\}$ where $X_1 = \{1, 2, 3\}, X_2 = \{3, 4, 5\}$ to cover all elements $\mc{U} = \{1, 2, 3, 4, 5\}$. The SetCover-DEC can be reduced to an instance of OI-DEC where the question is whether the MDP can be solved with 2 iterations of VI by adding at most two point options. The answer of OI-DEC is `Yes' (adding point options from $X_1$ and $X_2$ to $g$ will solve the problem), thus the answer of the SetCover-DEC is `Yes'. Here the set of initial states corresponds to the cover for the SetCover-DEC.
        }
        \label{fig:reduction}
    \end{figure}
    
\end{proof}

%% file: appendix.tex
\begin{appendices}
\section{Appendix: Inapproximability of MOMI}

%\subsection{Inapproximability of MOMI}

In this section we prove Theorem 4:

\setcounter{theorem}{3}
\begin{theorem} {\ }
    \begin{enumerate}
        \item MOMI is $\Omega(\log n)$ hard to approximate even for deterministic MDPs unless $P = NP$.
        \item MOMI$_{gen}$ is $2^{\log^{1-\epsilon}n}$-hard to approximate for any $\epsilon>0$ even for deterministic MDPs unless $NP \subseteq DTIME(n^{poly \log n})$.
        \item MOMI is $2^{\log^{1-\epsilon}n}$-hard to approximate for any $\epsilon>0$ unless $NP \subseteq DTIME(n^{poly \log n})$.
    \end{enumerate}
\end{theorem}
% First, we show Theorem 4.1 by a reduction from the set cover problem to MOMI with deterministic MDP.

% Next, we demonstrate Theorems 4.2 and 4.3. %that two natural variants of MOMI are $2^{\log ^{1 - \eps} n}$-hard-to-approximate for any $\eps > 0$. This bound can be thought of as approaching polynomial hardness of approximation: For $\eps = 0$ the hardness becomes $\Omega(2^{\log  n}) = \Omega(n)$. 
%We first show a $2^{\log ^{1 - \eps} n}$ hardness of approximation for \momichoose. % which asks, given a deterministic MDP, an input set of ``good'' (not necessarily) point options and a parameter $l$, what are the fewest options one can choose from this input set to guarantee that value iteration converges in at most $l$ iterations. This problem is to be distinguished from MOMI where \emph{any} point option can be chosen.
%\momichoose is a natural generalization of MOMI beyond point options. There is little sense in considering MOMI where one can choose any option since clearly the best option is the option whose policy is simply the optimal policy. Rather with \momichoose we model the scenario where we have a set of abstractions that we have learned from past experiences and we would like to choose the learned abstractions which are the best for our current problem.
%Next, we show that the above hardness of approximation holds for  \momi provided the input MDP is stochastic. 
%In \momistoch one is given a \emph{stochastic} MDP and a parameter $l$, and must choose as few point options as possible to reduce the convergence time of value iteration to $l$. 
For Theorems 4.2 and 4.3 we reduce our problem to the Min-Rep, problem, originally defined by \cite{kortsarz2001hardness}. Min-Rep is a variant of the better studied label cover problem \cite{dinur2004hardness} and has been integral to recent hardness of approximation results in network design problems \cite{dinitz2012label,bhattacharyya2012transitive}. Roughly, Min-Rep asks how to assign as few labels as possible to nodes in a bipartite graph such that every edge is ``satisfied.'' 

\ddef{Min-Rep}{\\
{\bf Given} a bipartite graph $G = (A \cup B, E)$ and alphabets $\Sigma_A$ and $\Sigma_B$ for the left and right sides of $G$ respectively. Each $e \in E$ has associated with it a set of pairs $\pi_{e} \subseteq \Sigma_A \times \Sigma_B$ which satisfy it. {\bf Return} a pair of assignments $\gamma_A : A \to \mathcal{P}(\Sigma_A)$ and $\gamma_B : B \to \mathcal{P}(\Sigma_B)$ such that for every $e=(A_i,B_j) \in E$  there exists an $(a, b) \in \pi_e$ such that $a \in \gamma_A(A_i)$ and $b \in \gamma_B(B_j)$. The objective is to minimize $\sum_{A_i \in A} |\gamma_A(A_i)| + \sum_{B_j \in B} |\gamma_B(B_j)|$.
}
\definecolor{OptionColor1}{HTML}{4285F4}
\definecolor{OptionColor2}{HTML}{FBBC05}
\definecolor{OptionColor3}{HTML}{34A853}
\definecolor{OptionColor4}{HTML}{EA4335}
\definecolor{OptionColor5}{HTML}{4F6367}
\definecolor{OptionColor6}{HTML}{D138BF}
\definecolor{OptionColor7}{HTML}{7494EA}

\tikzstyle{MREdgeStyle}=[thick, sloped]
\tikzstyle{SatAssignmentNode}=[color=OptionColor1]
\tikzstyle{MRNodeStyle}=[draw, circle, thick, minimum size=1cm]
\usetikzlibrary{positioning}

We illustrate a feasible solution to an instance of Min-Rep in \Cref{fig:minrep}. 
\begin{figure}[hb]
    \centering
    \begin{tikzpicture}
    % A
    \node [MRNodeStyle] (A1) at (0, 4) {$A_1$};
    \node [MRNodeStyle] (A2) at (0, 2) {$A_2$};
    % B 
    \node [MRNodeStyle] (B1) at (3, 4) {$B_1$};
    \node [MRNodeStyle] (B2) at (3, 2) {$B_2$};
        
    % E
    \draw[MREdgeStyle] (A1) -- (B1)  node[midway, above] {\tiny \color{OptionColor2} $(a_1, b_2)$};
    \draw[MREdgeStyle] (A1) -- (B2) node[midway, above] {\tiny \color{OptionColor3} $(a_2, b_3)$, \color{OptionColor4} $(a_3, b_1)$};
    \draw[MREdgeStyle] (A2) -- (B2) node[midway, above] {\tiny \color{OptionColor5} $(a_3, b_1)$};
    
    %Sat assignment
    \node [SatAssignmentNode, below=0cm of A1] {$a_1, a_2$};
    \node [SatAssignmentNode, below=0cm of A2] {$a_3$};
    \node [SatAssignmentNode, below=0cm of B1] {$b_2$};
    \node [SatAssignmentNode, below=0cm of B2] {$b_1, b_3$};
    
    \end{tikzpicture}
    \caption{An instance of Min-Rep with $\Sigma_A = \{a_1, a_2, a_3\}$ and $\Sigma_B = \{b_1, b_2, b_3\}$. Edge $e$ is labeled with pairs in $\pi_e$. Feasible solution $(\gamma_A, \gamma_b)$ illustrated where $\gamma_A(A_i)$ and $\gamma_B(B_j)$ below $A_i$ and $B_j$ in blue. Constraints colored to coincide with stochastic action colors in \Cref{fig:minrepstochMOMIred}.}
    \label{fig:minrep}
\end{figure}
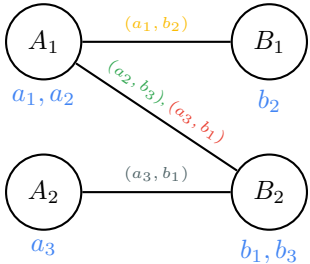

The crucial property of Min-Rep we use is that no polynomial-time algorithm can approximate Min-Rep well. 
Let $\tilde{n} = |A| + |B|$.
\begin{lemma}[\citeauthor{kortsarz2001hardness} \citeyear{kortsarz2001hardness}]\label{lem:minRepHard}
Unless $\text{NP} \subseteq \text{DTIME}(n^ {\poly \log n})$, Min-Rep admits no $2^{\log^{1 - \eps} \tilde{n}}$ polynomial-time approximation algorithm for any $\eps > 0$.
\end{lemma}
As a technical note, we emphasize that all relevant quantities in Min-Rep are polynomially-bounded. In Min-Rep we have $|\Sigma_A|, |\Sigma_B| \leq \tilde{n}^{c'}$ for constant $c'$. It immediately follows that $\sum_e |\pi_e| \leq n^c$ for constant $c$.

\subsection{Hardness of Approximation of \momi with Deterministic MDP}

\begin{proof}[Theorem 4.1 Proof]
    The optimization version of the set-cover problem cannot be approximated within a factor of $c \cdot \ln n$ by a polynomial-time algorithm unless P = NP~\cite{raz1997sub}. The set-cover optimization problem can be reduced to MOMI with a similar construction for a reduction from SetCover-DEC to OI-DEC. Here, the targeted minimization values of the two problems are equal: $P(\mc{C}) = |\mc{O}|$, and the number of states in OI-DEC is equal to the number of elements in the set cover on transformation.
    %\gdknote{Do we also need to say that the transformation is 1-1? Same size on both sides?}
    Assume there is a polynomial-time algorithm within a factor of $c \cdot \ln n$ approximation for MOMI where $n$ is the number of states in the MDP.
    Let SetCover$(\mc{U}, \mc{X})$ be an instance of the set-cover problem. We can convert the instance into an instance of MOMI$(M, 0, 2)$.
    Using the approximation algorithm, we get a solution $\mc{O}$ where $|\mc{O}| \leq c \ln n |\mc{O}^*|$, where $\mc{O}^*$ is the optimal solution.
    We construct a solution for the set cover $\mc{C}$ from the solution to the MOMI $\mc{O}$ (see the construction in the proof of Theorem~1).
    Because $|\mc{C}| = |\mc{O}|$ and $|\mc{C}^*| = |\mc{O}^*|$, where $\mc{C}^*$ is the optimal solution for the set cover, we get $|\mc{C}| = |\mc{O}| \leq c \ln n |\mc{O}^*| = c \ln n |\mc{C}^*|$. Thus, we acquire a $c \cdot \ln n$ approximation solution for the set-cover problem within polynomial time, something 
    only possible if P=NP.
    Thus, there is no polynomial-time algorithm with a factor of $c \cdot \ln n$ approximation for MOMI, unless P=NP.
\end{proof}

\subsection{Hardness of Approximation of \momichoose}
We now show our hardness of approximation of $2^{\log^{1 - \epsilon}n}$ for \momichoose,  Theorem 4.2.\footnote{We assume that $\mc{O}'$ is a ``good'' set of options in the sense that there exists some set $\mc{O}^* \subseteq \mc{O}'$ such that $L_{\epsilon, V_0}(\mc{O}^*) \leq \ell$. We also assume, without loss of generality, that $\eps < 1$ throughout this section; other values of $\eps$ can be handled by re-scaling rewards in our reduction.}
% --- MOMI Choose Problem ---
%\ddef{\momichoose}{\\
    %{\bf Given} a deterministic MDP $M$, a non-negative real-value $\epsilon$, a set of (not necessarily point) options $\mc{O}_i$ and an integer $\ell$, {\bf return} $\mc{O} \subseteq \mc{O}_i$ that minimizes $|\mc{O}|$ subject to $\mc{O} \subseteq \mc{O}_i$ and $L_{\epsilon, V_0}(\mc{O}) \leq \ell$.\footnote{We assume that $\mc{O}_i$ is a ``good'' set of options in the sense that there exists some set $\mc{O}^* \subseteq \mc{O}_i$ such that $L_{\epsilon, V_0}(\mc{O}^*) \leq \ell$. We also assume, without loss of generality, that $\eps < 1$ throughout this section; other values of $\eps$ can be handled by re-scaling rewards in our reduction.}
%}
%%%%%%%%

We start by describing our reduction from an instance of Min-Rep to an instance of \momichoose. The intuition behind our reduction is that we can encode choosing a label for a vertex in Min-Rep as choosing an option in our \momichoose instance. In particular, we will have a state for each edge in our Min-Rep instance and reward will propagate quickly to that state when value iteration is run only if the options corresponding to a satisfying assignment for that edge are chosen.

More formally, our reduction is as follows. Consider an instance of Min-Rep, $\MR$, given by $G = (A \cup B, E)$, $\Sigma_A$, $\Sigma_B$ and $\{\pi_e\}$. Our instance of \momichoose is as follows where $\gamma = 1$ and $l = 3$.\footnote{It is easy to generalize these results to $l \geq 4$ by replacing certain edges with paths.}

\begin{itemize}
    \item \textbf{State space} We have a single goal state $S_g$ along with states $S_g'$ and $S_g''$. For each edge $e$ we create a state $S_e$. Let $\text{Sat}_A(e)$ consist of all $a \in \Sigma_A$ such that $a$ is in some assignment in $\pi_e$. Define $\text{Sat}_B(e)$ symmetrically. For each edge $e \in E$ we create a set of $2 \cdot |\text{Sat}_A(e)|$ states, namely $S_{ea}$ and $S_{ea}'$ for every $a \in \text{Sat}_A(e)$. We do the same for $b \in \text{Sat}_B(e)$.
    \item \textbf{Actions and Transitions} We have a single action from $S_g'$ to $S_g$, a single action from $S_g''$ to $S_g'$. For each edge $e$ we have the following deterministic actions: Every $S_{ea}'$ has a single outgoing action to $S_{ea}$ for $a \in \text{Sat}_A(e)$; Every $S_{eb}$ has a single outgoing action to $S_{eb'}$ for $b \in \text{Sat}_B(e)$; Every $S_{ea}$ has an outgoing action to $S_{eb}$ if $(a,b) \in \pi_e$ and every $S_{eb}'$ has a single outgoing action to $S_g$; Lastly, we have a single action from $S_{ea}'$ to $S_g''$ for every $a \in \text{Sat}_A(e)$.
    \item \textbf{Reward} The reward of arriving in $S_g$ is $1$. The reward of arriving in every other state is $0$.
    \item \textbf{Option Set} Our option set $\mc{O}'$ is as follows. For each vertex $A_i \in A$ and each $a \in \Sigma_A$ we have an option $O(A_i,a)$: The initiation set of this option is every $S_e$ where $e$ is incident to $A_i$; The termination set of this option is every $S_{ea}$ where $A_i$ is incident to $e$; The policy of this option  takes the action from  $S_{ea}'$ to $S_{ea}$ when in $S_{ea}'$ and the action from $S_e$ to $S_{ea}'$ when in $S_e$. 
    
    Symmetrically, for every vertex $B_j \in B$ and each $b \in \Sigma_B$ we have an option $O(B_j, b)$: The initiation set of this option is every $S_{eb}$ where $e$ is incident to $B_j$; The termination set of this option is $S_g$; The policy of this option takes the action from $S_{eb}$ to $S_{eb}'$ when in $S_{eb}$ and from $S_{eb}'$ to $S_g$ when in $S_{eb}'$.
\end{itemize}
One should think of choosing option $O(v, x)$ as corresponding to choosing label $x$ for vertex $v$ in the input Min-Rep instance. Let $\momichoose(\MR)$ be the MDP output given instance \MR of Min-Rep and see \Cref{fig:minrepstochMOMIred} for an illustration of our reduction.  

\tikzstyle{MomiChooseAction}=[thick, ->]
\tikzstyle{MOMIMRRedNodeStyle}=[draw, circle, thick, minimum size=.9cm,inner sep=.1pt]
\begin{figure}
    \centering
    \scalebox{.75}{%Change fraction here to scale whole figure
    \begin{tikzpicture}
        % sg
        \node [MOMIMRRedNodeStyle] (Sg) at (4, 1.5) {$S_{g}$}; \node [draw, circle, minimum size = .7cm] () at (4, 1.5) {};
        \node [MOMIMRRedNodeStyle] (Sgp) at (0, 6) {$S_{g}'$};
        \node [MOMIMRRedNodeStyle] (Sgpp) at (-6, 1.5) {$S_{g}''$};
    
        % e1
        \node [MOMIMRRedNodeStyle] (Se1) at (-4, 4) {$S_{e_1}$};
        \node [MOMIMRRedNodeStyle] (Se1a1p) at (-2.5, 4) {$S_{e_1a_1}'$};
        \node [MOMIMRRedNodeStyle] (Se1a1) at (-1, 4) {$S_{e_1a_1}$};
        
        \node [MOMIMRRedNodeStyle] (Se1b2) at (1, 4) {$S_{e_1b_2}$};
        \node [MOMIMRRedNodeStyle] (Se1b2p) at (2.5, 4) {$S_{e_1b_2}'$};
        
        % e2
        \node [MOMIMRRedNodeStyle] (Se2) at (-4, 1.5) {$S_{e_2}$};
        \node [MOMIMRRedNodeStyle] (Se2a2p) at (-2.5, 2) {$S_{e_2a_2}'$};
        \node [MOMIMRRedNodeStyle] (Se2a2) at (-1, 2) {$S_{e_2a_2}$};
        \node [MOMIMRRedNodeStyle] (Se2a3p) at (-2.5, 1) {$S_{e_2a_3}'$};
        \node [MOMIMRRedNodeStyle] (Se2a3) at (-1, 1) {$S_{e_2a_3}$};
        
        \node [MOMIMRRedNodeStyle] (Se2b1p) at (2.5, 2) {$S_{e_2b1}'$};
        \node [MOMIMRRedNodeStyle] (Se2b1) at (1, 2) {$S_{e_2b1}$};
        \node [MOMIMRRedNodeStyle] (Se2b3p) at (2.5, 1) {$S_{e_2b3}'$};
        \node [MOMIMRRedNodeStyle] (Se2b3) at (1, 1) {$S_{e_2b3}$};
        
        % e3
        \node [MOMIMRRedNodeStyle] (Se3) at (-4, -1) {$S_{e_3}$};
        \node [MOMIMRRedNodeStyle] (Se3a3p) at (-2.5, -1) {$S_{e_3a_3}'$};
        \node [MOMIMRRedNodeStyle] (Se3a3) at (-1, -1) {$S_{e_3a_3}$};
        
        \node [MOMIMRRedNodeStyle] (Se3b1) at (1, -1) {$S_{e_3b_1}$};
        \node [MOMIMRRedNodeStyle] (Se3b1p) at (2.5, -1) {$S_{e_3b_1}'$};

        % Actions
        \draw[MomiChooseAction] (Se1) -- (Se1a1p);
        \draw[MomiChooseAction] (Se1a1p) -- (Se1a1);
        \draw[MomiChooseAction] (Se1a1) -- (Se1b2);
        \draw[MomiChooseAction] (Se1b2) -- (Se1b2p);
        \draw[MomiChooseAction] (Se1b2p) -- (Sg);
        
        \draw[MomiChooseAction] (Se2) -- (Se2a2p);
        \draw[MomiChooseAction] (Se2a2p) -- (Se2a2);
        \draw[MomiChooseAction] (Se2a2) -- (Se2b3);
        \draw[MomiChooseAction] (Se2b3) -- (Se2b3p);
        \draw[MomiChooseAction] (Se2b3p) -- (Sg);
        
        \draw[MomiChooseAction] (Se2) -- (Se2a3p);
        \draw[MomiChooseAction] (Se2a3p) -- (Se2a3);
        \draw[MomiChooseAction] (Se2a3) -- (Se2b1);
        \draw[MomiChooseAction] (Se2b1) -- (Se2b1p);
        \draw[MomiChooseAction] (Se2b1p) -- (Sg);
        
        \draw[MomiChooseAction] (Se3) -- (Se3a3p);
        \draw[MomiChooseAction] (Se3a3p) -- (Se3a3);
        \draw[MomiChooseAction] (Se3a3) -- (Se3b1);
        \draw[MomiChooseAction] (Se3b1) -- (Se3b1p);
        \draw[MomiChooseAction] (Se3b1p) -- (Sg);
        
        \path[black, ->, out=90, in=180, thick] (Sgpp) edge (Sgp);
        \path[black, ->, out=0, in=90, thick] (Sgp) edge (Sg);
        \path[black, ->, thick] (Se1a1p) edge (Sgpp);
        \path[black, ->, bend right, thick] (Se2a2p) edge (Sgpp);
        \path[black, ->, bend left, thick] (Se2a3p) edge (Sgpp);
        \path[black, ->, thick] (Se3a3p) edge (Sgpp);
        
        % Options
        \path[->, dashed, out=60, ultra thick, color=OptionColor1] (Se1) edge (Se1a1);
        
        \path[->, dashed, out=60, ultra thick, color=OptionColor2] (Se2) edge (Se2a2);
        
        \path[->, dashed, out =-60, in =-120, ultra thick, color=OptionColor3] (Se2) edge (Se2a3);
        
        \path[->, dashed, out =-60, in=-120, ultra thick, color=OptionColor4] (Se3) edge (Se3a3);
        
        \path[->, dashed, out =60, in=60, ultra thick, color=OptionColor5] (Se1b2) edge (Sg);
        
        \path[->, dashed, out =60, in=60, ultra thick, color=OptionColor6] (Se2b1) edge (Sg);
        \path[->, dashed, out =-60, in=-60, ultra thick, color=OptionColor6] (Se3b1) edge (Sg);

        \path[->, dashed, out =-60, in=-60, ultra thick, color=OptionColor7] (Se2b3) edge (Sg);

    \end{tikzpicture}
    }%end resize box
    \caption{Our \momichoose reduction applied to the Min-Rep problem in \Cref{fig:minrep}. $e_1 = (A_1, B_1)$, $e_2 = (A_1, B_2)$, $e_3 = (A_2, B_2)$. Actions given in solid lines and each option in $\mc{O}'$ represented in its own color as a dashed line from initiation to termination states. Notice that a single option goes from $S_{e3b1}$ and $S_{e_2b_1}$ to $S_g$.}
    \label{fig:minrepstochMOMIred}
\end{figure}
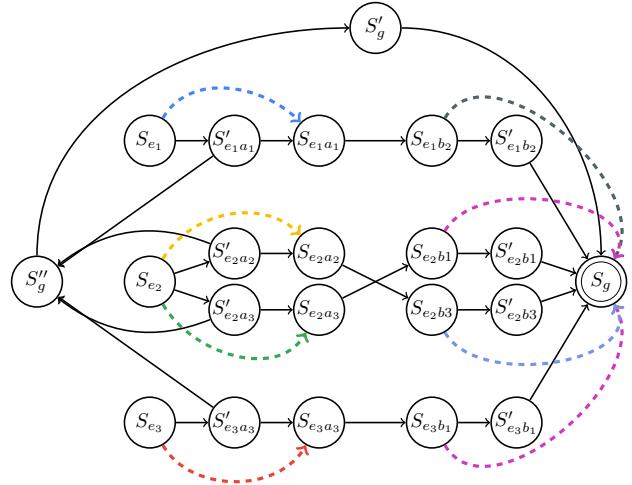

Let $\OPT_{\momichoose}$ be the value of the optimal solution to $\momichoose(\MR)$ and let $\OPT_\MR$ be the value of the optimal Min-Rep solution to \MR. The following lemmas demonstrates the correspondence between a \momichoose and Min-Rep solution.
\begin{lemma}\label{lem:MRRelateOptMOMIChoose} 
$\OPT_{\momichoose} \leq \OPT_\MR$
\end{lemma}
\begin{proof}
Given a solution $(\gamma_A, \gamma_B)$ to \MR, define $\mc{O}_{\gamma_A, \gamma_B} := \{O(v, x) : v \in V(G) \land (\gamma_A(v) = x \lor \gamma_B(v) = x) \}$ as the corresponding set of options. Let $\gamma_A^*$ and $\gamma_B^*$ be the optimal solutions to $\MR$ which is of cost $\OPT_\MR$. 

We now argue that $\mc{O}_{\gamma_A^*, \gamma_B^*}$ is a feasible solution to $\momichoose(\MR)$ of cost $\OPT_\MR$, demonstrating that the optimal solution to $\momichoose(\MR)$ has cost at most $\OPT_\MR$. To see this notice that by construction the \momichoose cost of $\mc{O}_{\gamma_A^*, \gamma_B^*}$ is exactly the Min-Rep cost of $(\gamma_A^*, \gamma_B^*)$. 

We need only argue, then, that $\mc{O}_{\gamma_A^*, \gamma_B^*}$ is feasible for $\momichoose(\MR)$ and do so now. The value of every state in $\momichoose(\MR)$ is $1$. Thus, we must guarantee that after 3 iterations of value iteration, every state has value $1$. However, without any options every state except each $S_e$ has value $1$ after 3 iterations of value iteration. Thus, it suffices to argue that $\mc{O}_{\gamma_A^*, \gamma_B^*}$ guarantees that every $S_e$ will have value $1$ after $3$ iterations of value iteration. Since $(\gamma_A^*, \gamma_B^*)$ is a feasible solution to \MR we know that for every $e = (A_i, B_j)$ there exists an $\bar{a} \in \gamma_A^*(A_i)$ and $\bar{b} \in \gamma_B^*(B_j)$ such that $(\bar{a}, \bar{b}) \in \pi_e$; correspondingly there are options $O(A_i, \bar{a}), O(B_j, \bar{b}) \in \mc{O}_{\gamma_A^*, \gamma_B^*}$. It follows that, given options $\mc{O}_{\gamma_A^*, \gamma_B^*}$ from, $S_e$ one can take option $O(A_i, \bar{a})$ then the action from $S_{e\bar{a}}$ to $S_{e\bar{b}}$ and then option $O(B_j, \bar{b})$ to arrive in $S_g$; thus, after 3 iterations of value iteration the value of $S_e$ is $1$. Thus, we conclude that after 3 iterations of value iteration every state has converged on its value.
\end{proof}

We now show that a solution to $\momichoose(\MR)$ corresponds to a solution to \MR. For the remainder of this section $\gamma^\mc{O}_A(A_i) : = \{a : O(A_i,a) \in \mc{O} \}$ and $\gamma^\mc{O}_B(B_j) : = \{b : O(B_j,b) \in \mc{O} \}$ is the Min-Rep solution corresponding to option set $\mc{O}$.
\begin{lemma}\label{lem:MRToMomiChoose}
For a feasible solution to $\momichoose(\MR)$, $\mc{O}$, we have $(\gamma^\mc{O}_A, \gamma^\mc{O}_B)$ is a feasible solution to \MR of cost $|\mc{O}|$.
\end{lemma}
\begin{proof}
Notice that by construction the Min-Rep cost of $(\gamma^{\mc{O}}_A, \gamma^{\mc{O}}_B)$ is exactly $|\mc{O}|$. Thus, we need only prove that $(\gamma^{\mc{O}}_A, \gamma^{\mc{O}}_B)$ is a feasible solution for \MR. 

We do so now. Consider an arbitrary edge $e = (A_i, B_j) \in E$; we wish to show that $(\gamma^{\mc{O}}_A, \gamma^{\mc{O}}_B)$ satisfies $e$. Since $\mc{O}$ is a feasible solution to $\momichoose(\MR)$ we know that after 3 iterations of value iteration every state must converge on its value. Moreover, notice that the value of every state in $\momichoose(\MR)$ is $1$. Thus, it must be the case that for every $S_e$ there exists a path of length 3 from $S_e$ to $S_g$ using either options or actions. The only such paths are those that take an option $O(A_i, a)$, then an action from $S_{ea}$ to $S_{eb}$ then option $O(B_j, b)$ where $(a, b) \in \pi_e$. It follows that $a \in \gamma^{\mc{O}}_A(A_i)$ and $b \in \gamma^{\mc{O}}_B(B_j)$. But since $(a, b) \in \pi_e$, we then know that $e$ is satisfied. Thus, every edge is satisfied and so $(\gamma^{\mc{O}}_A, \gamma^{\mc{O}}_B)$ is a feasible solution to \MR.
\end{proof}

%\begin{theorems}
%Unless $\text{NP} \subseteq \text{DTIME}(n^ {\poly \log n})$, no polynomial-time algorithm can $2^{\log^{1 - \eps} n }$-approximate \momichoose for any $\eps > 0$ where $n$ is the number of states in the MDP.
%\end{theorems}
\begin{proof}[Theorem 4.2 Proof]
Assume $\text{NP} \not \subseteq \text{DTIME}(n^ {\poly \log n})$ and for the sake of contradiction that there exists an $\eps > 0$ for which polynomial-time algorithm $\mc{A}_{\momichoose}$ can $2^{\log^{1 - \eps} n }$-approximate \momichoose. We use $\mc{A}_{\momichoose}$ to $2^{\log^{1 - \eps'} \tilde{n}}$ approximate Min-Rep for a fixed constant $\eps' > 0$ in polynomial-time, thereby contradicting \Cref{lem:minRepHard}. Again, $\tilde{n}$ is the number of vertices in the graph of the Min-Rep instance.

We begin by noting that the relevant quantities in $\momichoose(\MR)$ are polynomially-bounded. Notice that the number of states $n$ in the MDP in $\momichoose(\MR)$ is at most $O(\tilde{n}^2 |\Sigma_A| |\Sigma_B|) = \tilde{n}^c$ for some fixed constant $c$ by the aforementioned assumption that $\Sigma_A$ and $\Sigma_B$ are polynomially-bounded in $\tilde{n}$.\footnote{It is also worth noticing that since we create at most $O(\tilde{n}|\Sigma_A| + \tilde{n}|\Sigma_B|)$ options, the total number of options in $\mc{O}'$ is at most polynomial in $\tilde{n}$.}

Our polynomial-time approximation algorithm to approximate instance \MR of Min-Rep is as follows: Run $\mc{A}_{\momichoose}$ on $\momichoose(\MR)$ to get back option set $\mc{O}$. Return $(\gamma^\mc{O}_A, \gamma^\mc{O}_B)$ as defined above as our solution to \MR.

We first argue that our algorithm is polynomial-time in $\tilde{n}$. However, notice that for each vertex, we create a polynomial number of states. Thus, the number of states in $\momichoose(\MR)$ is polynomially-bounded in $\tilde{n}$ and so $\mc{A}_{\momichoose}$ runs in time polynomial in $\tilde{n}$. A polynomial runtime of our algorithm immediately follows.

We now argue that our algorithm is a $2^{\log^{1 - \eps'} \tilde{n} }$-approximation for Min-Rep for some $\eps' > 0$. Applying \Cref{lem:MRToMomiChoose}, the approximation of $\mc{A}_{\momichoose}$ and then \Cref{lem:MRRelateOptMOMIChoose}, we have that $(\gamma^\mc{O}_A, \gamma^\mc{O}_B)$  is a feasible solution for \MR with cost
\begin{align*}
\cost_{\text{Min-Rep}}(\gamma^\mc{O}_A, \gamma^\mc{O}_B) & = |\mc{O}|\\ &\leq 2^{\log^{1-\eps} n } \OPT_{\momichoose} \\
&\leq 2^{\log^{1-\eps} n } \OPT_\MR
\end{align*}

Thus, $(\gamma^\mc{O}_A, \gamma^\mc{O}_B)$ is a $2^{\log^{1-\eps}n}$ approximation for the optimal Min-Rep solution where $n$ is the number of states in the MDP of $\momichoose(\MR)$. Now recalling that $n \leq \tilde{n}^c$ for fixed constant $c$. We therefore have that $(\gamma^\mc{O}_A, \gamma^\mc{O}_B)$ is a $2^{\log^{1-\eps} \tilde{n}^c} = 2^{c^{1- \eps} \log^{1-\eps} \tilde{n}} \leq c' \cdot 2^{ \log^{1-\eps} \tilde{n}}$ approximation for a constant $c'$. Choosing $\eps$ sufficiently small, we have that $c' \cdot 2^{ \log^{1-\eps} \tilde{n}} \leq 2^{ \log^{1-\eps'} \tilde{n}}$ for sufficiently large $\tilde{n}$. 

Thus, our polynomial-time algorithm is a $2^{\log^{1 - \eps'} \tilde{n} }$-approximation for Min-Rep for $\eps' > 0$, thereby contradicting \Cref{lem:minRepHard}. We conclude that \momichoose cannot be $2^{\log^{1- \epsilon} n}$-approximated.
\end{proof}

\subsection{Hardness of Approximation of \momi with Stochastic MDP}
We now show our hardness of approximation of $2^{\log^{1 - \epsilon}n}$ for \momi, Theorem 4.3. We will notably use the stochasticity of the input MDP to show this result.% for \momistoch which is as follows.
\footnote{We may assume without loss of generality $\eps < .5$ throughout this section; rewards in our reduction can be re-scaled to handle larger $\eps$.}

% --- MOMI StochasticProblem ---
%\ddef{\momistoch}{\\
%    {\bf Given} a stochastic MDP $M$, a non-negative real-value $\epsilon$, and an integer $\ell$, {\bf return} $\mc{O}$ that minimizes $|\mc{O}|$ subject to $\mc{O} \subseteq \mc{O}_p$ and $L_{\epsilon, V_0}(\mc{O}) \leq \ell$.\footnote{We may assume without loss of generality $\eps < .5$ throughout this section; rewards in our reduction can be re-scaled to handle larger $\eps$.}
%}

We begin by describing our reduction from an instance of Min-Rep to an instance of \momi. The intuition behind our reduction is as follows. As in our reduction for \momichoose we will have vertex for each edge in our Min-Rep instance and reward will propagate quickly to that vertex when value iteration is run only if the options corresponding to a satisfying assignment for that edge are chosen. The challenge, however, is that since our options are now only point options (whereas in \momichoose they were arbitrary options) it seems that we can no longer constrain a solution to choose options exactly corresponding to a feasible Min-Rep solution. 

To solve this issue we critically use stochasticity. Whether or not a given edge in a Min-Rep is satisfied is an or of ands: A fixed edge is satisfied when \emph{one} of its satisfying assignments is met (an or) and a given satisfying assignment is met when both endpoints have the right labels (an and). We will exploit the fact that the value of a state in an MDP is a max over actions to encode the ``or'' in Min-Rep and we will use the fact that in a stochastic MDP the value of a (state, action) pair is the sum over states to encode the ``and'' in Min-Rep.

More formally, our reduction is as follows. Consider instance \MR of Min-Rep given by $G = (A \cup B, E)$, $\Sigma_A$, $\Sigma_B$ and $\{\pi_e\}$. Our instance of \momi is as follows where $\gamma = 1$ and $l = 2$.\footnote{It is easy to generalize these results to $l \geq 3$ by replacing edges with paths.}

\begin{itemize}
    \item \textbf{State space} We have a goal state $S_i$ for each $A_i \in A$. Again, let $\text{Sat}_A(e)$ consist of all $a \in \Sigma_A$ such that $a$ is in some assignment in $\pi_e$. For each $A_i \in A$ and $a \in \text{Sat}_A(e)$ we will we add to our MDP states $S_{ia}$ and $S_{ia}'$. We symmetrically do the same for all states in $\Sigma_B$. For each $e \in E$ we will also add a state $S_e$.\footnote{It is not hard to see that this construction can be modified so that we have only a single goal state if need be; we need only set every $S_i$ and $S_j$ to be the same state. We assume multiple goal states for ease of exposition.}
    \item \textbf{Actions and Transitions} Every $S_{ia}$ state has a single action to $S_{ia}'$ and every $S_{ia}$ state has a single action to $S_i$. The same symmetrically holds for states from a $B_j \in B$. Every $S_e$ for $e = (A_i, B_j)$ has $|\pi_{(A_i, B_j)}|$ actions associated with it, namely $\{\alpha_{(a,b)}\}$ where $(a, b) \in \pi_{(A_i, B_j)}$. Action $\alpha_{(a,b)}$ has a probability $.5$ of transitioning to state $S_{ia}$ and a probability $.5$ of transitioning to state $S_{jb}$.
    \item \textbf{Reward} The reward of arriving in any $S_i$ or $S_j$ for $A_i \in A$ or $B_j \in B$ is $1$ and $0$ for every other state. 
\end{itemize}
Notice that no point options have $S_{e}$ as an initialization state since any such option would have a $.5$ probability of never terminating (and we assume our options always terminate). See \Cref{fig:minrepstochMOMIred} for an illustration of our reduction.   One should think of choosing a point option from $S_{ia}$ to $S_{i}$ as corresponding to choosing label $a$ for $A_i$ in the input Min-Rep instance. The same holds for label $b$ for $B_j$ and choosing a point option from $S_{jb}$ to $S_{j}$. Let $\momi(\MR)$ be the \momi instance output by our reduction given instance \MR of Min-Rep. 

\tikzstyle{MomiStochAction}=[thick, ->]
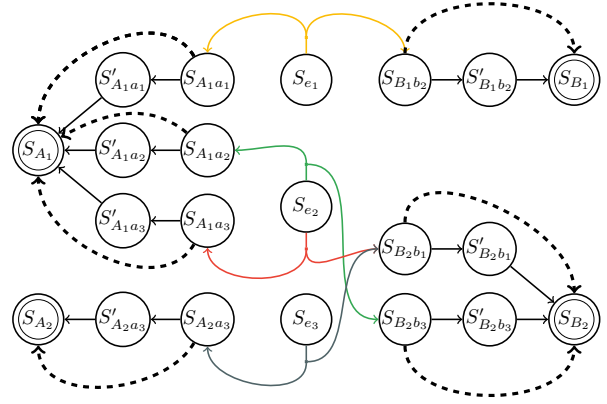
\begin{figure}
    \centering
    \scalebox{.75}{
    \begin{tikzpicture}
        % A1
        \node [MOMIMRRedNodeStyle] (SA1) at (-4, 3) {$S_{A_1}$}; \node [draw, circle, minimum size = .7cm] () at (-4, 3) {};
        \node [MOMIMRRedNodeStyle] (SA1a1p) at (-2.5, 4.25) {$S_{A_1a_1}'$};
        \node [MOMIMRRedNodeStyle] (SA1a1) at (-1, 4.25) {$S_{A_1a_1}$};
        \node [MOMIMRRedNodeStyle] (SA1a2p) at (-2.5, 3) {$S_{A_1a_2}'$};
        \node [MOMIMRRedNodeStyle] (SA1a2) at (-1, 3) {$S_{A_1a_2}$};
        \node [MOMIMRRedNodeStyle] (SA1a3p) at (-2.5, 1.75) {$S_{A_1a_3}'$};
        \node [MOMIMRRedNodeStyle] (SA1a3) at (-1, 1.75) {$S_{A_1a_3}$};
        
        % A2
        \node [MOMIMRRedNodeStyle] (SA2) at (-4, 0) {$S_{A_2}$}; \node [draw, circle, minimum size = .7cm] () at (-4, 0) {};
        \node [MOMIMRRedNodeStyle] (SA2a3p) at (-2.5, 0) {$S_{A_2a_3}'$};
        \node [MOMIMRRedNodeStyle] (SA2a3) at (-1, 0) {$S_{A_2a_3}$};

        %B1
        \node [MOMIMRRedNodeStyle] (SB1b2) at (2.5, 4.25) {$S_{B_1b_2}$};
        \node [MOMIMRRedNodeStyle] (SB1b2p) at (4, 4.25) {$S_{B_1b_2}'$};
        \node [MOMIMRRedNodeStyle] (SB1) at (5.5, 4.25) {$S_{B_1}$}; \node [draw, circle, minimum size = .7cm] () at (5.5, 4.25) {};
        
        %B2
        \node [MOMIMRRedNodeStyle] (SB2b1) at (2.5, 1.25) {$S_{B_2b_1}$};
        \node [MOMIMRRedNodeStyle] (SB2b1p) at (4, 1.25) {$S_{B_2b_1}'$};
        \node [MOMIMRRedNodeStyle] (SB2b3) at (2.5, 0) {$S_{B_2b_3}$};
        \node [MOMIMRRedNodeStyle] (SB2b3p) at (4, 0) {$S_{B_2b_3}'$};     
        \node [MOMIMRRedNodeStyle] (SB2) at (5.5, 0) {$S_{B_2}$}; \node [draw, circle, minimum size = .7cm] () at (5.5, 0) {};
        
        %Edge nodes
        \node [MOMIMRRedNodeStyle] (Se1) at (.75, 4.25) {$S_{e_1}$}; \node [inner sep = 0cm, draw, rectangle, color=OptionColor2, fill=OptionColor2] (Se1p) at (.75, 5) {};
        \node [MOMIMRRedNodeStyle] (Se2) at (.75, 2) {$S_{e_2}$}; \node [inner sep = 0cm, draw, rectangle, color=OptionColor3, fill=OptionColor3] (Se2p) at (.75, 2.75) {}; \node [inner sep = 0cm, draw, rectangle,color=OptionColor4, fill=OptionColor4] (Se2pp) at (.75, 1.25) {};
        \node [MOMIMRRedNodeStyle] (Se3) at (.75, 0) {$S_{e_3}$}; \node [inner sep = 0cm, draw, rectangle,color=OptionColor5, fill=OptionColor5] (Se3p) at (.75, -.75) {};
        
        %Deterministic actions
        \draw[MomiStochAction] (SA1a1p) -- (SA1);
        \draw[MomiStochAction] (SA1a2p) -- (SA1);
        \draw[MomiStochAction] (SA1a3p) -- (SA1);
        \draw[MomiStochAction] (SA1a1) -- (SA1a1p);
        \draw[MomiStochAction] (SA1a2) -- (SA1a2p);
        \draw[MomiStochAction] (SA1a3) -- (SA1a3p);
        
        \draw[MomiStochAction] (SB1b2p) -- (SB1);
        \draw[MomiStochAction] (SB1b2) -- (SB1b2p);
        
        \draw[MomiStochAction] (SB2b1) -- (SB2b1p);
        \draw[MomiStochAction] (SB2b3) -- (SB2b3p);
        \draw[MomiStochAction] (SB2b1p) -- (SB2);
        \draw[MomiStochAction] (SB2b3p) -- (SB2);
        
        \draw[MomiStochAction] (SA2a3) -- (SA2a3p);
        \draw[MomiStochAction] (SA2a3p) -- (SA2);
        
        %Stochastic actions

        \draw[MomiStochAction, -, color = OptionColor2] (Se1) -- (Se1p);
        \path[black, ->, out=90, in=90, thick, color = OptionColor2] (Se1p) edge (SA1a1);
        \path[black, ->, out=90, in=90, thick, color = OptionColor2] (Se1p) edge (SB1b2);

        \draw[MomiStochAction, -, color = OptionColor3] (Se2) -- (Se2p);
        \path[black, ->, out=90, in=0, thick, color = OptionColor3] (Se2p) edge (SA1a2);
        \path[black, ->, out=0, in=180, thick, color = OptionColor3] (Se2p) edge (SB2b3);
        
        \draw[MomiStochAction, -, color = OptionColor4] (Se2) -- (Se2pp);
        \path[black, ->, out=-90, in=-90, thick, color = OptionColor4] (Se2pp) edge (SA1a3);
        \path[black, ->, out=-90, in=180, thick, color = OptionColor4] (Se2pp) edge (SB2b1);
        
        \draw[MomiStochAction, -, color = OptionColor5] (Se3) -- (Se3p);
        \path[black, ->, out=-90, in=-90, thick, color = OptionColor5] (Se3p) edge (SA2a3);
        \path[black, ->, out=0, in=180, thick, color = OptionColor5] (Se3p) edge (SB2b1);

        %Point options
        \path[->, dashed, out=120, in=90, ultra thick] (SA1a1) edge (SA1);
        \path[->, dashed, out=120, in=90, ultra thick] (SA1a1) edge (SA1);
        \path[->, dashed, out=140, in=30, ultra thick] (SA1a2) edge (SA1);
        \path[->, dashed, out=-120, in=-90, ultra thick] (SA1a3) edge (SA1);
        
        \path[->, dashed, out=-120, in=-90, ultra thick] (SA2a3) edge (SA2);
        
        \path[->, dashed, out=90, in=90, ultra thick] (SB1b2) edge (SB1);
        
        \path[->, dashed, out=90, in=90, ultra thick] (SB2b1) edge (SB2);
        \path[->, dashed, out=-90, in=-90, ultra thick] (SB2b3) edge (SB2);

    \end{tikzpicture}
    }%end scale box
    \caption{Our \momi reduction applied to the Min-Rep problem in \Cref{fig:minrep}. $e_1 = (A_1, B_1)$, $e_2 = (A_1, B_2)$, $e_3 = (A_2, B_2)$. Stochastic options colored according to the pair in $\pi_e$ to which they correspond, branching into the two states in which they arrive with equal probability. Deterministic action given as solid black arcs. Possible point options given as dashed arcs.}
    \label{fig:minrepstochMOMIred}
\end{figure}

We now demonstrate that our reduction allows us to show that \momi cannot be $2^{\log^{1-\eps} n}$-approximated for any $\eps > 0$. Let $\OPT_\momi$ be the value of the optimal solution to $\momi(\MR)$ and let $\OPT_\MR$ be the value of the optimal Min-Rep solution to \MR. The following lemmas demonstrates the correspondence between a \momi and Min-Rep solution.

\begin{lemma}\label{lem:MRRelateOpt} 
$\OPT_\momi \leq \OPT_\MR$
\end{lemma}
\begin{proof}
Our proof translates between point options in our reduction and assignments in the input Min-Rep instance in the natural way. Given a solution $(\gamma_A, \gamma_B)$ to \MR, define $\mc{O}_{\gamma_A, \gamma_B}$ as consisting of all point options from $S_{ia}$ to $S_i$ if $a \in \gamma_A(A_i)$ and all points options from $S_{jb}$ to $S_j$ if $b \in \gamma_B(B_j)$. Let $\gamma_A^*$ and $\gamma_B^*$ be the optimal solutions to $\MR$ which is of cost $\OPT_\MR$. 

We claim that $\mc{O}_{\gamma_A^*, \gamma_B^*}$ is a feasible solution to $\momi(\MR)$ of cost $\OPT_\MR$, demonstrating that the optimal solution to $\momi(\MR)$ has cost at most $\OPT_\MR$. To see this notice that by construction the \momi cost of $\mc{O}_{\gamma_A^*, \gamma_B^*}$ is exactly the Min-Rep cost of $\gamma_A^*, \gamma_B^*$. 

We need only argue, then, that $\mc{O}_{\gamma_A^*, \gamma_B^*}$ is feasible for $\momi(\MR)$ and do so now. Notice that the value of every state in $\momi$ is $1$. Thus, we must guarantee that after 2 iterations of value iteration, every state has value $1$. However, without any options every state except for $S_e$ where $e \in E$ has value $1$ after 2 iterations of value iteration. Thus, it suffices to argue that $\mc{O}_{\gamma_A^*, \gamma_B^*}$ guarantees that every $S_e$ will have value $1$ after $2$ iterations of value iteration. Since $(\gamma_A^*, \gamma_B^*)$ is a feasible solution to \MR we know that for every $e = (A_i, B_j)$ there exists $\bar{a} \in \gamma_A^*(A_i)$ and $\bar{b} \in \gamma_B^*(B_j)$ such that $(\bar{a}, \bar{b}) \in \pi_e$; correspondingly there is some action from $S_e$ with a $.5$ probability of resulting in state $S_{i\bar{a}}$ and a $.5$ probability of resulting in state $S_{j\bar{b}}$ where $\mc{O}_{\gamma_A^*, \gamma_B^*}$ has a point option from $S_{i\bar{a}}$ to $S_i$ and a point options from $S_{j\bar{b}}$ to $S_j$. That is, $V_1(S_{i\bar{a}}) = 1$ and $V_1(S_{j\bar{b}}) = 1$. Thus, after one iteration of value iteration the values of $S_{i\bar{a}}$ and $S_{j\bar{b}}$ are both 1 and so after two iterations of value iteration the value of $S_e$ is
\begin{align*}
V_{2}(S_e) & = \max_{\alpha_{(a,b)}} .5 \cdot (V_1(S_{ia})) + .5 \cdot( V_1(S_{jb}))\\
& \geq .5 \cdot (V_1(S_{i\bar{a}})) + .5 \cdot( V_1(S_{j\bar{b}}))\\
& = 1.
\end{align*}
Thus, $V_{2}(S_e) = 1$ for every $S_e$ and so we conclude that after two iterations of value iteration every state has converged on its value.
\end{proof}

We now show that a solution to $\momi(\MR)$ corresponds to a solution to \MR. For the remainder of this section let $\gamma^\mc{O}_A(A_i) : = \{a : O(S_{ja}, S_j) \in \mc{O} \}$ and  $\gamma^\mc{O}_B(B_j) : = \{b : O(S_{jb}, S_j) \in \mc{O} \}$ where for the remainder of this section $O(S, S')$ stands for a point option with initiation state $S$ and termination state $S'$.
\begin{lemma}\label{lem:MRToMomiStoch}
For any feasible solution  $\mc{O}$ to $\momi(\MR)$ we have $(\gamma^\mc{O}_A, \gamma^\mc{O}_B)$ is a feasible solution to \MR of cost $|\mc{O}|$.
\end{lemma}
\begin{proof}
Notice that by construction the Min-Rep cost of $(\gamma^{\mc{O}}_A, \gamma^{\mc{O}}_B)$ is exactly $|\mc{O}|$. Thus, we need only prove that $(\gamma^{\mc{O}}_A, \gamma^{\mc{O}}_B)$ is a feasible solution for \MR. 

We do so now. Consider an arbitrary edge $e = (A_i, B_j) \in E$; we wish to show that $(\gamma^{\mc{O}}_A, \gamma^{\mc{O}}_B)$ satisfies $e$. Since $\mc{O}$ is a feasible solution we know that after two iterations of value iteration every state must converge on its value (up to an $\epsilon$ factor which we can ignore by our above assumption that $\eps < .5$). Moreover, notice that the value of every state in $\momi(\MR)$ is $1$. Thus, it must be the case that for every $S_e$ we have $V_2(S_e) = 1$ for $e = (A_i, B_j)$. It follows, then, that there is some action $\alpha_{(\bar{a},\bar{b})}$ where $(\bar{a}, \bar{b}) \in \pi_{(A_i, B_j)}$ such that 
\begin{align*}
1 = V_2(S_e) = .5 \cdot (V_1(S_{i\bar{a}})) + .5 \cdot( V_1(S_{j\bar{b}})).
\end{align*}
Since the value of every state is at most $1$, it follows that $V_1(S_{i\bar{a}}) = V_1(S_{j\bar{b}}) = 1$. However, since $V_1(S_{i\bar{a}})$ and $V_1(S_{j\bar{b}})$ are both two hops from the only goal reachable from them ($S_i$ and $S_j$ respectively) it must be the case that there is some point option from $S_{i\bar{a}}$ to $S_i$ and $S_{j\bar{b}}$ to $S_j$. Thus, by definition of $(\gamma^{\mc{O}}_A, \gamma^{\mc{O}}_B)$ we then have $\bar{a} \in \gamma^{\mc{O}}_A$ and $\bar{b} \in \gamma^{\mc{O}}_B$. Since $(\bar{a}, \bar{b}) \in \pi_{(A_i, B_j)}$ it follows that arbitrary edge $e =(A_i, B_j)$ is satisfied. Thus, every edge in $E$ is satisfied and so $(\gamma^{\mc{O}}_A, \gamma^{\mc{O}}_B)$ is a feasible solution for \MR.
\end{proof}

Finally, we conclude the hardness of approximation of \momi.
%\begin{theorem}
%Unless $\text{NP} \subseteq \text{DTIME}(n^ {\poly \log n})$, no polynomial-time algorithm can $2^{\log^{1 - \eps} n }$-approximate \momi for any $\eps > 0$ where $n$ is the number of states in the MDP.
%\end{theorem}
\begin{proof}[Theorem 4.3 Proof]
Assume $\text{NP} \not \subseteq \text{DTIME}(n^ {\poly \log n})$ and for the sake of contradiction that there exists an $\eps > 0$ for which a polynomial-time algorithm $\mc{A}_\momi$ can $2^{\log^{1 - \eps} n }$-approximate \momi. We use $\mc{A}_\momi$ to $2^{\log^{1 - \eps'} \tilde{n}}$ approximate Min-Rep for a fixed constant $\eps' > 0$ in polynomial-time in $\tilde{n}$, thereby contradicting \Cref{lem:minRepHard}. Again, $\tilde{n}$ is the number of vertices in the graph of the Min-Rep instance.

We begin by noting that the relevant quantities in $\momi(\MR)$ are polynomially-bounded. Let $\tilde{n} := |A| + |B|$ be the number of vertices in our \MR instance. Notice that the number of states in the MDP, $n$, in our $\momi(\MR)$ instance is at most $O(\tilde{n} + 2|A||\Sigma_A|  + |B||\Sigma_B| + |E|) \leq \tilde{n}^c$ for some fixed constant $c$ by the aforementioned assumption that $\Sigma_A$ and $\Sigma_B$ are polynomially-bounded in $\tilde{n}$.\footnote{It is worth noting, also, that since we create at most $\sum_e |\pi_e|$ actions for any state, the number of total actions in our MDP is at most polynomial in $\tilde{n}$.}

Our polynomial-time approximation algorithm to approximate instance \MR of Min-Rep is as follows: Run $\mc{A}_\momi$ on $\momi(\MR)$ to get back option set $\mc{O}$. Return $(\gamma^\mc{O}_A, \gamma^\mc{O}_B)$ as defined above as our solution to \MR.

We first argue that our algorithm is polynomial time in $\tilde{n}$. For each vertex in \MR, we create a polynomial number of states and actions. Thus, the number of states in $\momi(\MR)$ is polynomially-bounded in $\tilde{n}$ and so $\mc{A}_\momi$ runs in time polynomial in $\tilde{n}$. A polynomial runtime of our algorithm immediately follows.

We now argue that our algorithm is a $2^{\log^{1 - \eps'} \tilde{n} }$-approximation for Min-Rep for some $\eps' > 0$. Applying \Cref{lem:MRToMomiStoch}, the approximation of $\mc{A}_\momi$ and then \Cref{lem:MRRelateOpt}, we have that the Min-Rep cost of $(\gamma^\mc{O}_A, \gamma^\mc{O}_B)$ is
\begin{align*}
\cost_{\text{Min-Rep}}(\gamma^\mc{O}_A, \gamma^\mc{O}_B) & = |\mc{O}|\\ &\leq 2^{\log^{1-\eps} n } \OPT_\momi \\
&\leq 2^{\log^{1-\eps} n } \OPT_\MR
\end{align*}

Thus, $(\gamma^\mc{O}_A, \gamma^\mc{O}_B)$ is a $2^{\log^{1-\eps}n}$ approximation for the optimal Min-Rep solution where $n$ is the number of states in the MDP of $\momi(\MR)$. Now recalling that $n \leq \tilde{n}^c$ for fixed constant $c$. We therefore have that $(\gamma^\mc{O}_A, \gamma^\mc{O}_B)$ is a $2^{\log^{1-\eps} \tilde{n}^c} = 2^{c^{1- \eps} \log^{1-\eps} \tilde{n}} \leq c' \cdot 2^{ \log^{1-\eps} \tilde{n}}$ approximation for a constant $c'$. Choosing $\eps$ sufficiently small, we have that $c' \cdot 2^{ \log^{1-\eps} \tilde{n}} \leq 2^{ \log^{1-\eps'} \tilde{n}}$ for sufficiently large $\tilde{n}$. 

Thus, our polynomial-time algorithm is a $2^{\log^{1 - \eps'} \tilde{n} }$-approximation for Min-Rep for $\eps' > 0$, thereby contradicting \Cref{lem:minRepHard}. We conclude that \momi cannot be $2^{\log^{1- \epsilon} n}$-approximated.
\end{proof}

\subsection{\mimoalg{}}
In this subsection we show the following theorem (we show Theorem 5 later):

\setcounter{theorem}{5}
\begin{theorem}
    \mimoalg{} has following properties:
    \begin{enumerate}
        \item \mimoalg{} runs in polynomial time.
        \item If the MDP is deterministic, it has a bounded suboptimality of $O(\log^* k)$.
        \item The number of iterations to solve the MDP using the acquired options is upper bounded by $P(\mc{C})$.
    \end{enumerate}
\end{theorem}

% --- Theorem ---
\begin{theorems}
    \mimoalg{} runs in polynomial time.
\end{theorems}
\begin{proof}
    Each step of the procedure runs in polynomial time.
    
    (1) Solving an MDP takes polynomial time. To compute $d$ we need to solve MDPs at most $|\mc{S}|$ times. Thus, it runs in polynomial time.
    
    (2) The approximation algorithm we deploy for solving the asymmetric-k center which runs in polynomial time \cite{archer2001two}. % Solving an LP.
    Because the procedure by \namecite{archer2001two} terminates immediately after finding a set of options which guarantees the suboptimality bounds, it tends to find a set of options smaller than $k$. 
    In order to use the rest of the options effectively within polynomial time, we use a procedure $\text{Expand}$ to greedily add a few options at once until it finds all $k$ options. We enumerate all possible set of options of size $r = \lceil \log k \rceil$ (if $|\mc{O}| + \log k > k$ then we set $r = k - |\mc{O}|$) and add a set of options which minimizes $\ell$ (breaking ties randomly) to the option set $\mc{O}$. We repeat this procedure until $|\mc{O}| = k$.
    This procedure runs in polynomial time. The number of possible option set of size $r$ is $_rC_n = O(n^r) = O(k)$. We repeat this procedure at most $\lceil k / \log k \rceil$ times, thus the total computation time is bounded by $O(k^2 / \log k)$.
    
    (3) Immediate.
    
    Therefore, \mimoalg{} runs in polynomial time.
\end{proof}

Before we show that it is sufficient to consider a set of options with its terminal state set to the goal state of the MDP.

% Proposition: <description>
\begin{lemma}
\label{th:terminal}
    There exists an optimal option set for MIMO and MOMI with all terminal state set to the goal state.
\end{lemma}

\begin{proof}
    Assume there exists an option with terminal state set to a state other than the goal state in the optimal option set $\mc{O}$. By triangle inequality, swapping the terminal state to the goal state will monotonically decrease $d(s, g)$ for every state. By swapping every such option we can construct an option set $\mc{O}'$ with $L_{\epsilon, V_0}(\mc{O}') \leq L_{\epsilon, V_0}(\mc{O})$. 
\end{proof}

Lemma \label{th:terminal} imply that discovering the best option set among option sets with their terminal state fixed to the goal state is sufficient to find the best option set in general.
Therefore, our algorithms seek to discover options with termination state fixed to the goal state.

Using the option set acquired, the number of iterations to solve the MDP is bounded by $P(\mc{C})$.
To prove this we first generalize the definition of the distance function to take a state and a set of states as arguments $d_{\epsilon}: \mc{S} \times 2^{\mc{S}} \rightarrow \mathbb{N}$.
Let $d_\epsilon(s, \mc{C})$ the number of iterations for $s$ to converge $\epsilon$-optimal if every state $s' \in \mc{C}$ has converged to $\epsilon$-optimal: $d_\epsilon(s, \mc{C}) := \min (d'_\epsilon(s), 1 + d'_\epsilon(s, \mc{C})) - 1$. As adding an option will never make the number of iterations larger, 

\begin{lemma}
    \begin{equation}
    \label{eq:set-distance}
        d(s, \mc{C}) \leq \min_{s' \in \mc{C}} d(s, s').
    \end{equation}
\end{lemma}

Using this, we show the following proposition.

% Theorem: <description>
\begin{theorems}
    The number of iterations to solve the MDP using the acquired options is upper bounded by $P(\mc{C})$.
\end{theorems}

\begin{proof}
    % As $d$ satisfies the triangle inequality, $\min_{c \in \mc{C}} d(s, c) \geq d(s, \mc{C})$.
    $P(\mc{C}) = \max_{s \in S} \min_{c \in \mc{C}} d(s, c) \geq \max_{s \in S} d(s, \mc{C}) = L_{\epsilon, V_0}(\mc{O})$ (using Equation \ref{eq:set-distance}). Thus $P(\mc{C})$ is an upper bound for $L_{\epsilon, V_0}(\mc{O})$.
    % With a set of point options from every $c \in C$ to the goal state, $c \in C$ reaches their optimal values with one iteration. Thus, with $max_{s \in S} min_{c \in C} d(s, c)$ more iterations, every $s \in S$ reaches $\epsilon$-optimal values. Thus with $P(C) + 1$ iterations every state reaches $\epsilon$-optimal value.
\end{proof}

The reason why $P(\mc{C})$ does not always give us the exact number of iterations is because adding two options starting from $s_1, s_2$ may make the convergence of $s_0$ faster than $d(s_0, s_1)$ or $d(s_0, s_2)$.
\noindent{Example:} Figure \ref{fig:stochastic} is an example of such an MDP.
From $s_0$ it may transit to $s_1$ and $s_2$ with probability 0.5 each. Without any options, the value function converges to exactly optimal value for every state with 3 steps.
Adding an option either from $s_1$ or $s_2$ to $g$ does not shorten the iteration for $s_0$ to converge. However, if we add two options from $s_1$ and $s_2$ to $g$, $s_0$ converges within 2 steps, thus the MDP is solved with 2 steps.

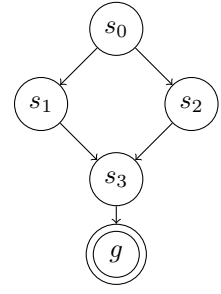
\begin{figure}
    \centering
    \begin{tikzpicture}
        \node [draw, circle] (s0) at (2, 3) {$s_0$};
        \node [draw, circle] (s1) at (1, 2) {$s_1$};
        \node [draw, circle] (s2) at (3, 2) {$s_2$};
        \node [draw, circle] (s3) at (2, 1) {$s_3$};
        \node [draw, circle, minimum size=0.8cm] (g) at (2, 0) {};
        \node [draw, circle, minimum size=0.5cm] at (2, 0) {$g$};

        \draw[->] (s0) -- (s1);
        \draw[->] (s0) -- (s2);
        \draw[->] (s1) -- (s3);
        \draw[->] (s2) -- (s3);
        \draw[->] (s3) -- (g);

    \end{tikzpicture}
    \caption{An example of an MDP where $d(s, \mc{C}) < \min_{s' \in \mc{C}} d(s, s')$. Here the transition induced by the optimal policy is stochastic, thus from $s_0$ one may go to $s_1$ and $s_2$ by probability 0.5 each. Either adding an option from $s_1$ or $s_2$ to $g$ does not make the convergence faster, but adding both makes it faster.}
    \label{fig:stochastic}
\end{figure}

The equality of the statement \ref{eq:set-distance} holds if the MDP is deterministic. That is, $d(s, \mc{C}) = \min_{s' \in \mc{C}} d(s, s')$ for deterministic MDP.

\begin{theorems}
    \item If the MDP is deterministic, it has a bounded suboptimality of $O(\log^* k)$.
\end{theorems}

\begin{proof}
    First we show $P(\mc{C}^*) = L_{\epsilon, V_0}(\mc{O}^*)$ for deterministic MDP.
    From $d(s, \mc{C}) = \min_{s' \in \mc{C}} d(s, s')$,
    $P(\mc{C}^*) = \max_{s \in \mc{S}} \min_{c \in \mc{C}^*} d(s, c) = \max_{s \in \mc{S}} d(s, \mc{C}^*) = L_{\epsilon, V_0}(\mc{O}^*)$.
    
    The asymmetric $k$-center solver guarantees that the output $\mc{C}$ satisfies $P(\mc{C}) \leq c(\log^* k + O(1)) P(\mc{C}^*)$ where $n$ is the number of nodes \cite{archer2001two}.
    Let MIMO$(M, \epsilon, k)$ be an instance of MIMO.
    We convert this instance to an instance of asymmetric $k$-center AsymKCenter$(\mc{U}, d, k)$, where $|\mc{U}| = |\mc{S}|$.
    By solving the asymmetric $k$-center with the approximation algorithm, we get a solution $\mc{C}$ which satisfies $P(\mc{C}) \leq c(\log^* k + O(1)) P(\mc{C}^*)$.
    Thus, the output of the algorithm $\mc{O}$ satisfies $L_{\epsilon, V_0}(\mc{O}) = P(\mc{C}) \leq c(\log^* k + O(1)) P(\mc{C}^*) = c(\log^* k + O(1)) L_{\epsilon, V_0}(\mc{O}^*)$.
    Thus, $L_{\epsilon, V_0}(\mc{O}) \leq c(\log^* k + O(1)) L_{\epsilon, V_0}(\mc{O}^*)$ is derived.
\end{proof}

\begin{proposition}[Greedy Strategy]
    Let an option set $\mc{O}$ be a set of point option constructed by greedily adding one point option which minimizes the number of iterations. 
    An improvement $L_{\epsilon, V_0}(\emptyset) - L_{\epsilon, V_0}({\mc{O}})$ by the greedy algorithm can be arbitrary small (i.e. 0) compared to the optimal option set. 
\end{proposition}
\begin{proof}
    We show by the example in a shortest-path problem in Figure \ref{fig:bound}.
    The MDP can be solved within $4$ iterations without options: $L_{\epsilon, V_0}(\emptyset) = 4$.
    With an optimal option set of size $k=2$ the MDP can be solved within $2$ iterations: $L_{\epsilon, V_0}(\mc{O}^*) = 2$ (an initiation state of each option in optimal option set is denoted by $*$ in the Figure).
    On the other hand, a greedy strategy may not improve $L$ at all.
    No single point option does not improve $L$. Let's say we picked a point option from $s_1$ to $g$. Then, there is no single point option we can add to that option to improve $L$ in the second iteration. Therefore, the greedy procedure returns $\mc{O}$ which has $L_{\epsilon, V_0}(\emptyset) - L_{\epsilon, V_0}(\mc{O}) = 0$.
    Therefore, $(L_{\epsilon, V_0}(\emptyset) - L_{\epsilon, V_0}(\mc{O})) / (L_{\epsilon, V_0}(\emptyset) - L_{\epsilon, V_0}(\mc{O}^*))$ can be arbitrary small non-negative value (i.e. 0).
    
    \begin{figure}[h]
    \centering
    \begin{tikzpicture}
        \node [draw, circle, minimum size=0.5cm] (x1) at (1, 0) {};
        \node [draw, circle, minimum size=0.5cm] (x2) at (2, 0) {};
        \node [draw, circle, minimum size=0.5cm] (x3) at (3, 0) {$*$};
        
        \node [draw, circle, minimum size=0.5cm] (x4) at (4, 1) {$s_1$};
        \node [draw, circle, minimum size=0.5cm] (x5) at (4, 0) {};
        \node [draw, circle, minimum size=0.5cm] (x6) at (4, -1) {};

        \node [draw, circle, minimum size=0.5cm] (y1) at (-1, 0) {};
        \node [draw, circle, minimum size=0.5cm] (y2) at (-2, 0) {};
        \node [draw, circle, minimum size=0.5cm] (y3) at (-3, 0) {$*$};
        
        \node [draw, circle, minimum size=0.5cm] (y4) at (-4, 1) {};
        \node [draw, circle, minimum size=0.5cm] (y5) at (-4, 0) {};
        \node [draw, circle, minimum size=0.5cm] (y6) at (-4, -1) {};
        
        \node [draw, circle, minimum size=0.8cm] (g) at (0, 0) {$g$};
        \node [draw, circle, minimum size=0.5cm] at (0, 0) {};

        \draw[->] (x6) -- (x3);
        \draw[->] (x5) -- (x3);
        \draw[->] (x4) -- (x3);
        \draw[->] (x3) -- (x2);
        \draw[->] (x2) -- (x1);
        \draw[->] (x1) -- (g);
        
        \draw[->] (y6) -- (y3);
        \draw[->] (y5) -- (y3);
        \draw[->] (y4) -- (y3);
        \draw[->] (y3) -- (y2);
        \draw[->] (y2) -- (y1);
        \draw[->] (y1) -- (g);
        
    \end{tikzpicture}
    \caption{Example of MIMO where the improvement of a greedy strategy can be arbitrary small compared to the optimal option set.}
    \label{fig:bound}
\end{figure}
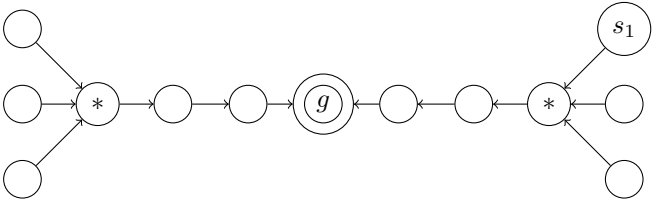
    
\end{proof}

%%%%%%%%%%%%%%%%%%%%%%%%%%%%%%%%%%%%%%%%%%%%%%%%%%%%%%%%
\subsection{\momialg{}}

In this subsection we show the following theorem:
\setcounter{theorem}{4}
\begin{theorem}
    \momialg{} has the following properties:
    \begin{enumerate}
        \item \momialg{} runs in polynomial time.
        \item It guarantees that the MDP is solved within $\ell$ iterations using the option set acquired by \momialg{} $\mc{O}$.
        \item If the MDP is deterministic, the option set is at most $O(\log k)$ times larger than the smallest option set possible to solve the MDP within $\ell$ iterations.
        % \item If the MDP is deterministic, the option set is at most $\max_{s \in \mc{S}} X_s$ times larger than the smallest option set possible to solve the MDP within $\ell$ iterations. %, where $\Delta$ is the maximum number of states which can reach to a certain state within $\ell - 1$ steps ($\Delta = \max_{s \in \mc{S}} X_s$).
    \end{enumerate}
\end{theorem}

\begin{theorems}
    \momialg{} runs in polynomial time.
\end{theorems}
\begin{proof}
    Each step of the procedure runs in polynomial time.
    
    (1) Solving an MDP takes polynomial time \cite{littman1995complexity}. To compute $d$ we need to solve MDPs at most $|\mc{S}|$ times. Thus, it runs in polynomial time.
    
    (4) We solve the set cover using a polynomial time approximation algorithm \cite{chvatal1979greedy} which runs in $O(n^3)$, thus run in polynomial time.
    
    (2), (3), and (5) Immediate.
\end{proof}

\begin{theorems}
    \momialg{} guarantees that the MDP is solved within $\ell$ iterations using the option set $\mc{O}$.
\end{theorems}
\begin{proof}
    A state $s \in X^{+}_{g}$ reaches optimal within $\ell$ steps by definition. For every state $s \in \mc{S} \setminus X^{+}_{g}$, the set cover guarantees that we have $X_{s'} \in \mc{C}$ such that $d(s, s') < \ell$. As we generate an option from $s'$ to $g$, $s'$ reaches to optimal value with 1 step. Thus, $s$ reaches to $\epsilon$-optimal value within $d(s, s') + 1 \leq \ell$.
    Therefore, every state reaches $\epsilon$-optimal value within $\ell$ steps.
\end{proof}

\begin{theorems}
    If the MDP is deterministic, the option set is at most $\max_{s \in \mc{S}} X_s$ times larger than the smallest option set possible to solve the MDP within $\ell$ iterations.
\end{theorems}
\begin{proof}
    Using a suboptimal algorithm by \namecite{chvatal1979greedy} we get $\mc{C}$ such that $|\mc{C}| \leq O(\log n) |\mc{C}^*|$  where $\Delta$ is the maximum size of subsets in $\mc{X}$.
    Thus, $|\mc{O}| = |\mc{C}| \leq O(\log n) |\mc{C}^*| = O(\log n) |\mc{O}^*|$.
\end{proof}

\section*{Appendix: Experiments}

We show the figures for experiments.
Figure \ref{fig:fourroom-viz} shows the options found by solving MIMO optimally/suboptimally in four room domain.
Figure \ref{fig:9x9grid-viz} shows the options in 9x9 grid domain.

\begin{figure*}
    \centering
    \newcommand{\figsize}{0.15}
    \subfloat[optimal $k = 1$]{\includegraphics[width=\figsize\textwidth]{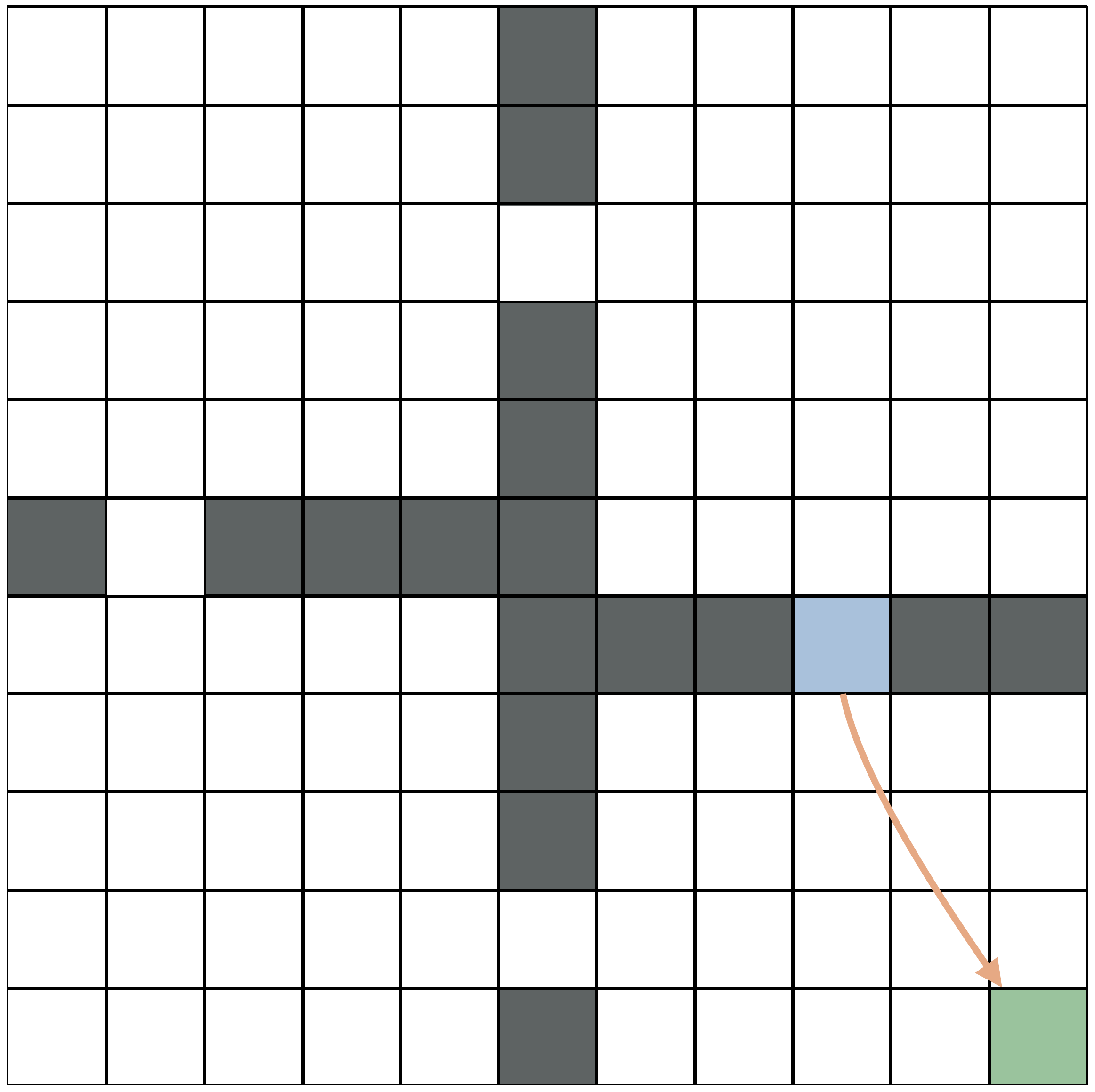}}\hspace{1mm}
    \subfloat[optimal $k = 2$]{\includegraphics[width=\figsize\textwidth]{figures/pdf_images/fourrooms/four_room_2_opt.pdf}}\hspace{1mm}
    \subfloat[optimal $k = 3$]{\includegraphics[width=\figsize\textwidth]{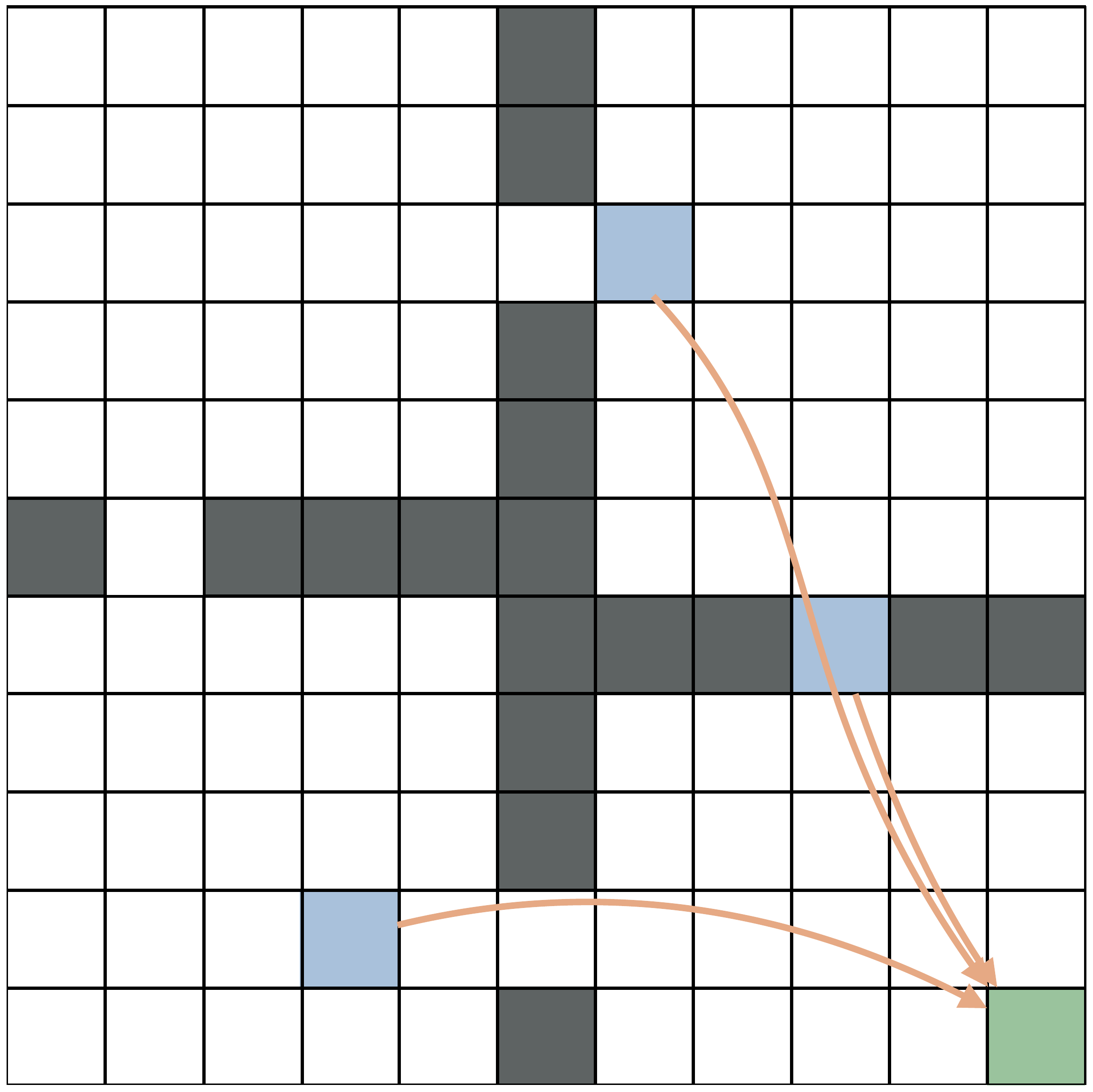}}\hspace{1mm}
    \subfloat[optimal $k = 4$]{\includegraphics[width=\figsize\textwidth]{figures/pdf_images/fourrooms/four_room_4_opt.pdf}}

    \subfloat[approx. $k = 1$]{\includegraphics[width=\figsize\textwidth]{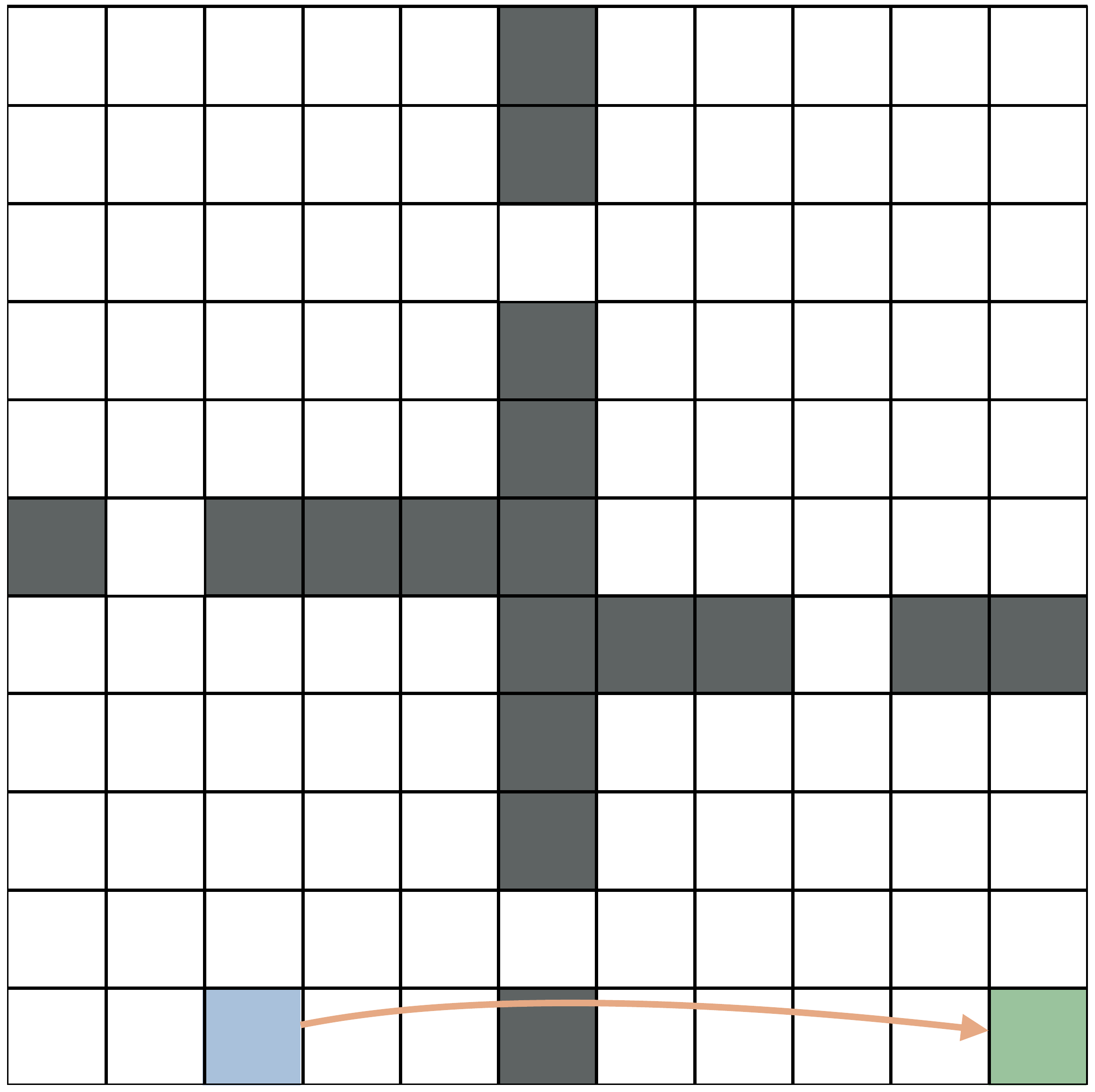}}\hspace{1mm}
    \subfloat[approx. $k = 2$]{\includegraphics[width=\figsize2\textwidth]{figures/pdf_images/fourrooms/four_room_2sub_opt.pdf}}\hspace{1mm}
    \subfloat[approx. $k = 3$]{\includegraphics[width=\figsize\textwidth]{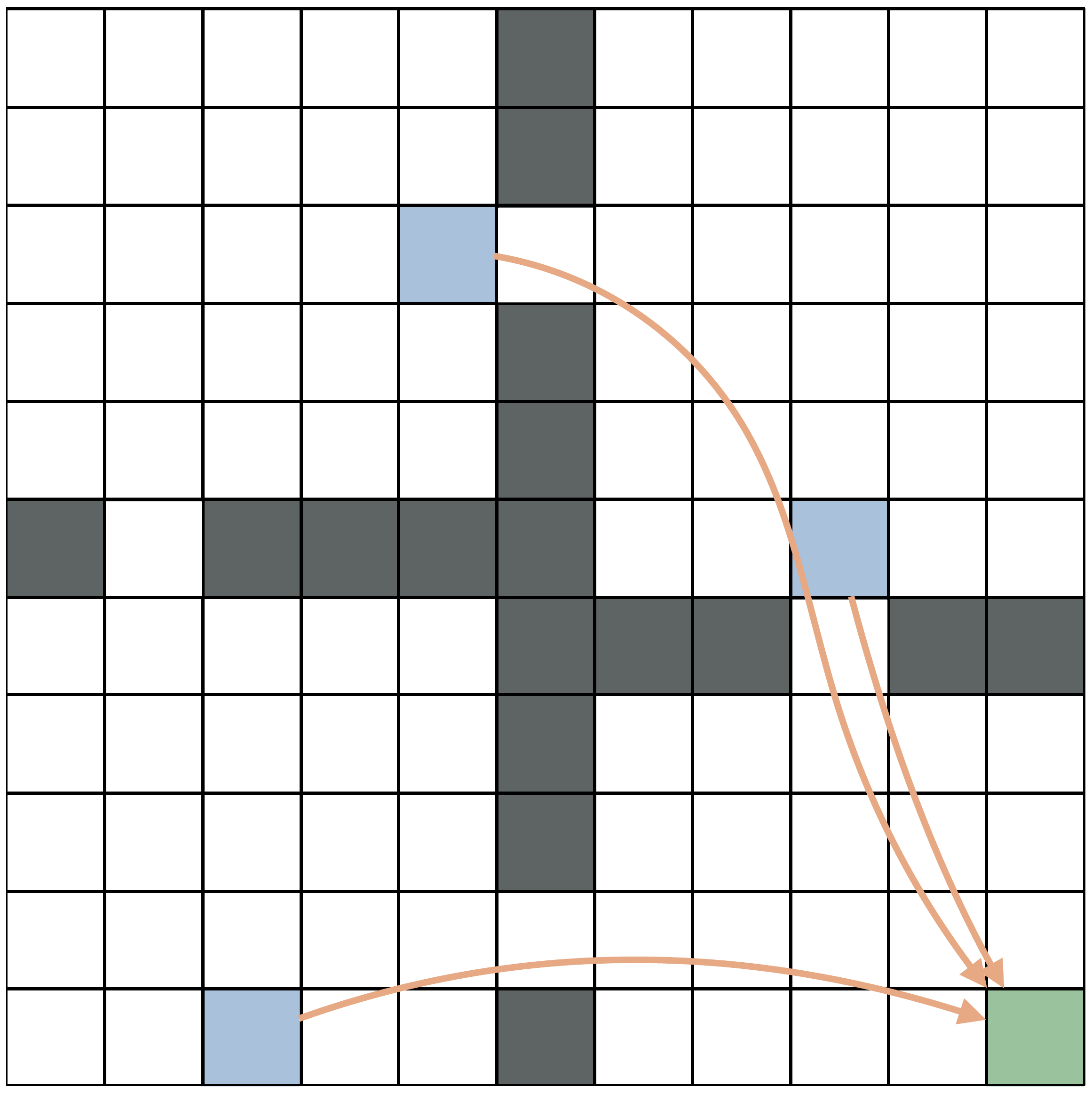}}\hspace{1mm}
    \subfloat[approx. $k = 4$]{\includegraphics[width=\figsize\textwidth]{figures/pdf_images/fourrooms/four_room_4sub_opt.pdf}}
    
    \subfloat[Betweenness]{\includegraphics[width=\figsize\textwidth]{figures/pdf_images/fourrooms/four_room_bet.pdf}}\hspace{1mm}
    \subfloat[Eigenoptions]{\includegraphics[width=\figsize\textwidth]{figures/pdf_images/fourrooms/four_room_eigen.pdf}} 

    \caption{Comparison of the optimal point options vs. options generated by the approximation algorithm \mimoalg{}. We observed that the approximation algorithm is similar to that of optimal options. Note that optimal option set is not unique: there can be multiple optimal option set, and we are visualize one of them returned by the solver.}
    \label{fig:fourroom-viz}
\end{figure*}

\begin{figure*}
    \centering
    \subfloat[optimal $k = 1$]{\includegraphics[width=0.2\textwidth]{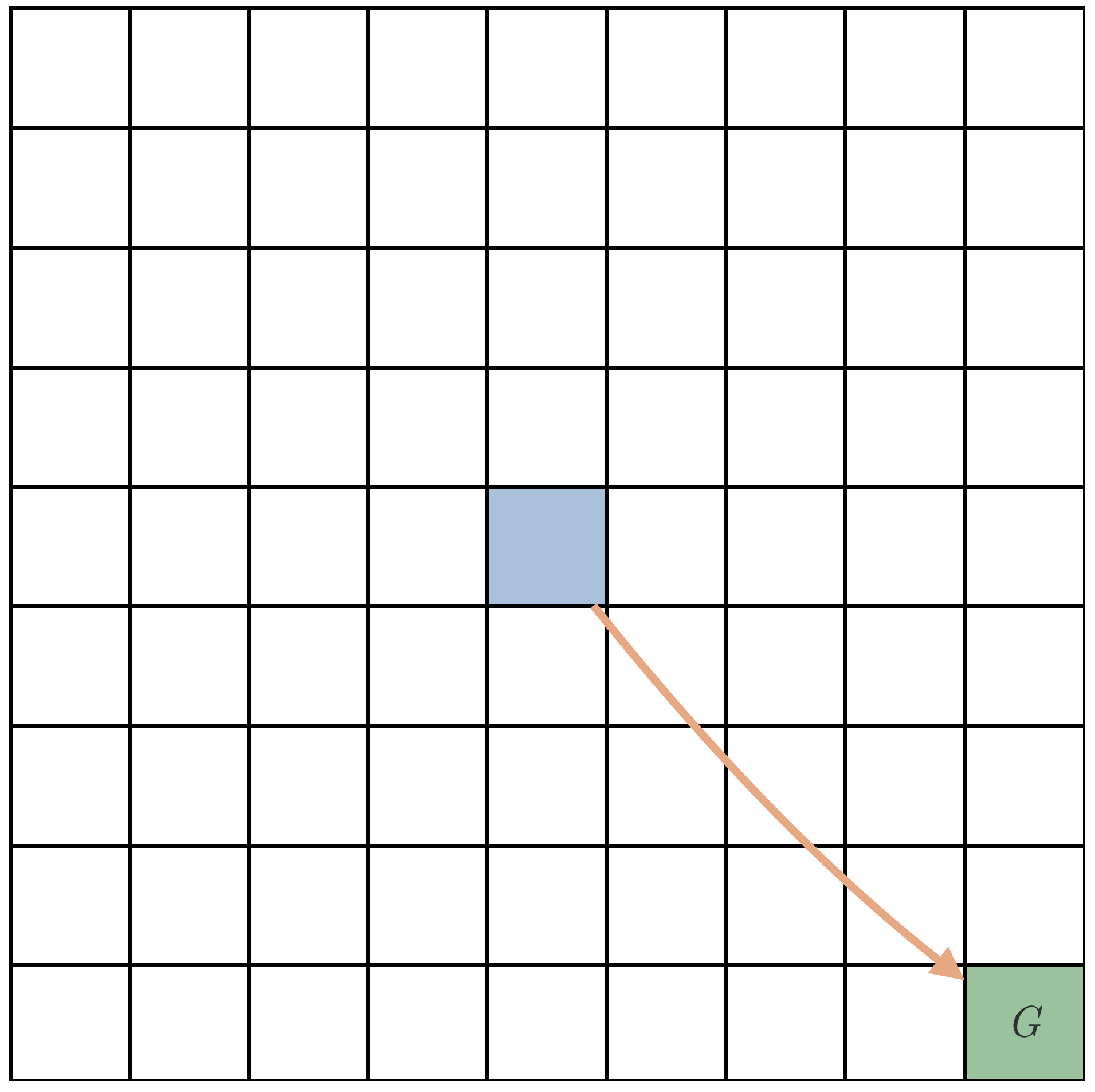}} \hspace{1mm}
    \subfloat[optimal $k = 2$]{\includegraphics[width=0.2\textwidth]{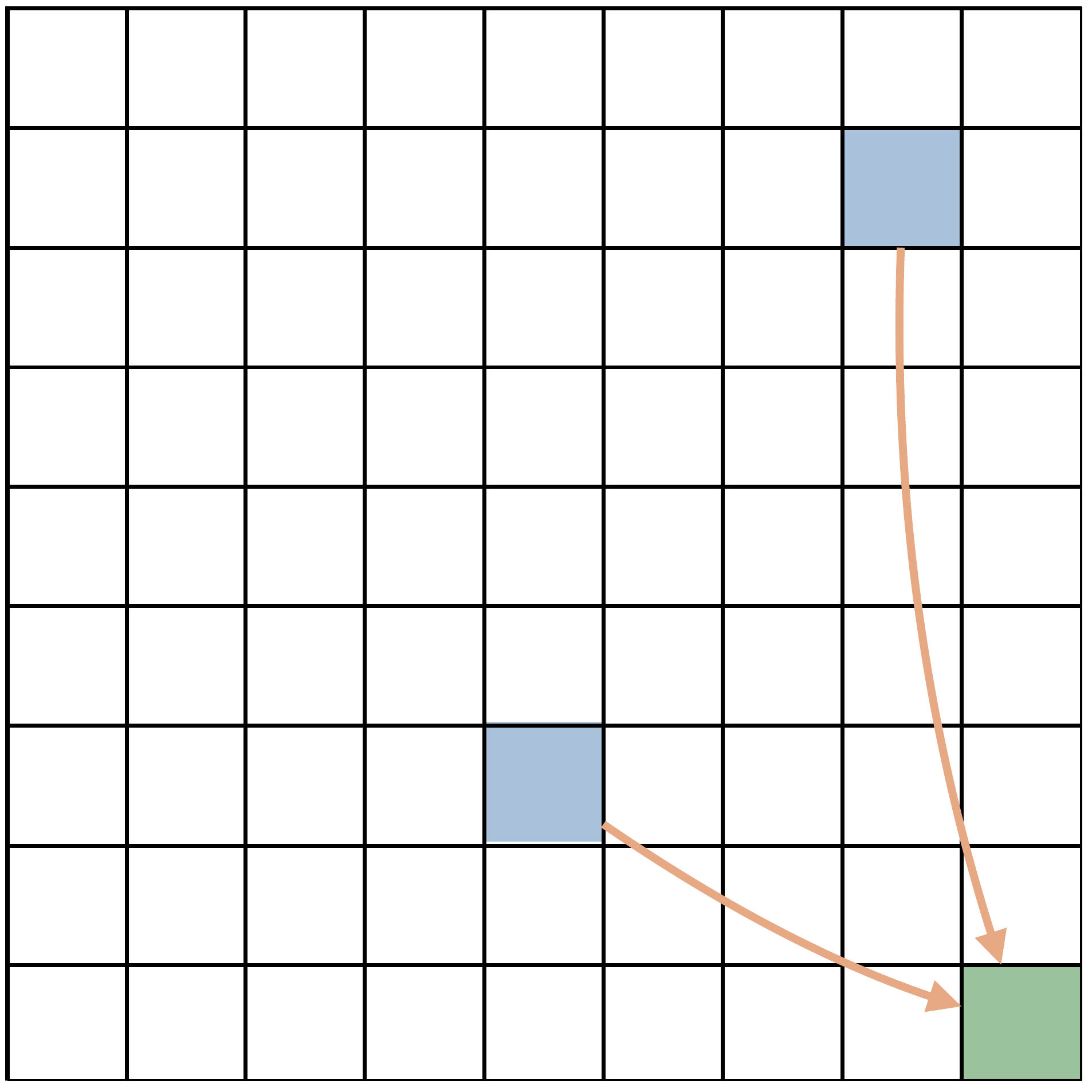}} \hspace{1mm}
    \subfloat[optimal $k = 3$]{\includegraphics[width=0.2\textwidth]{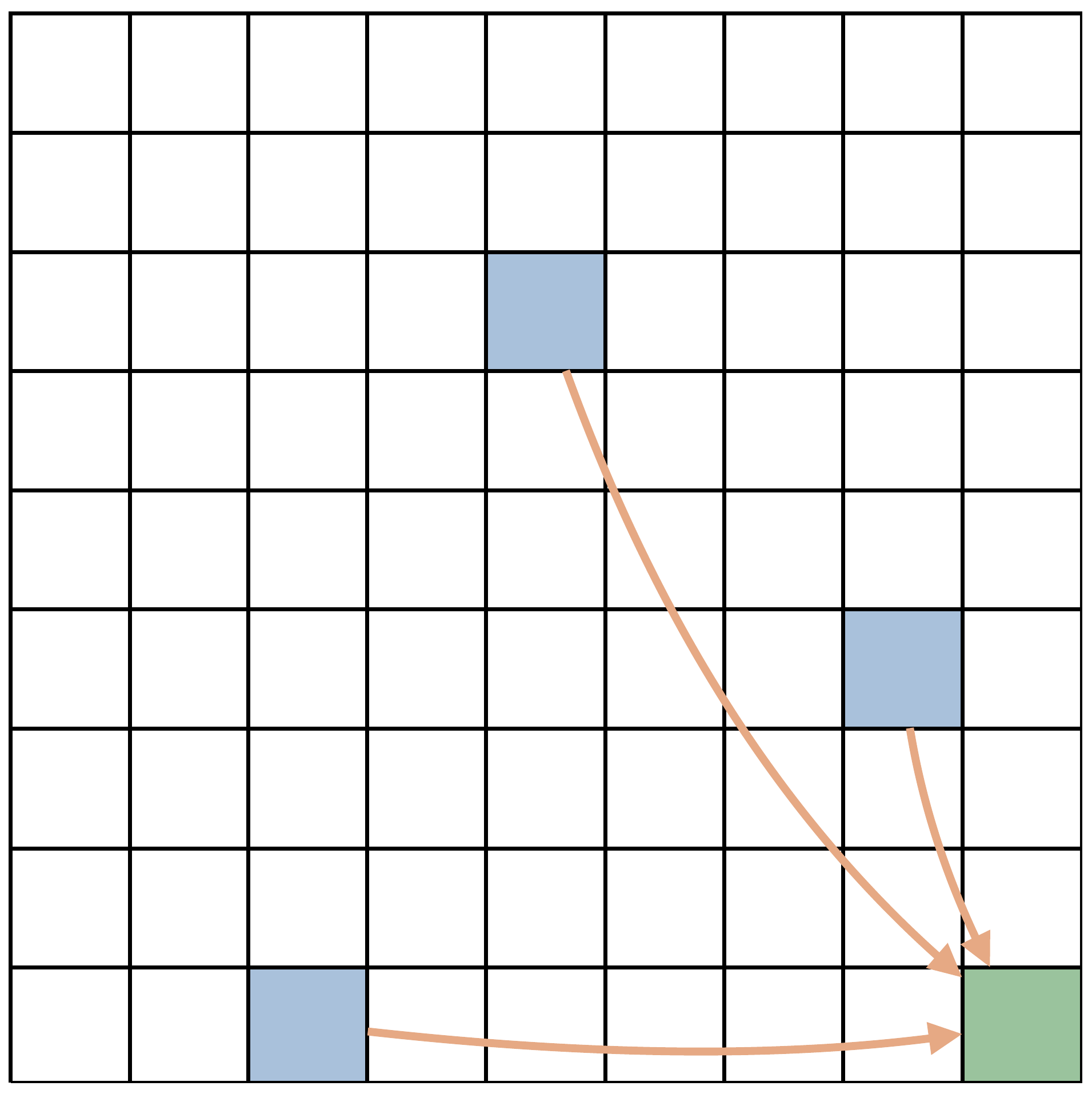}}\hspace{1mm}
    
    \subfloat[approx. $k = 1$]{\includegraphics[width=0.2\textwidth]{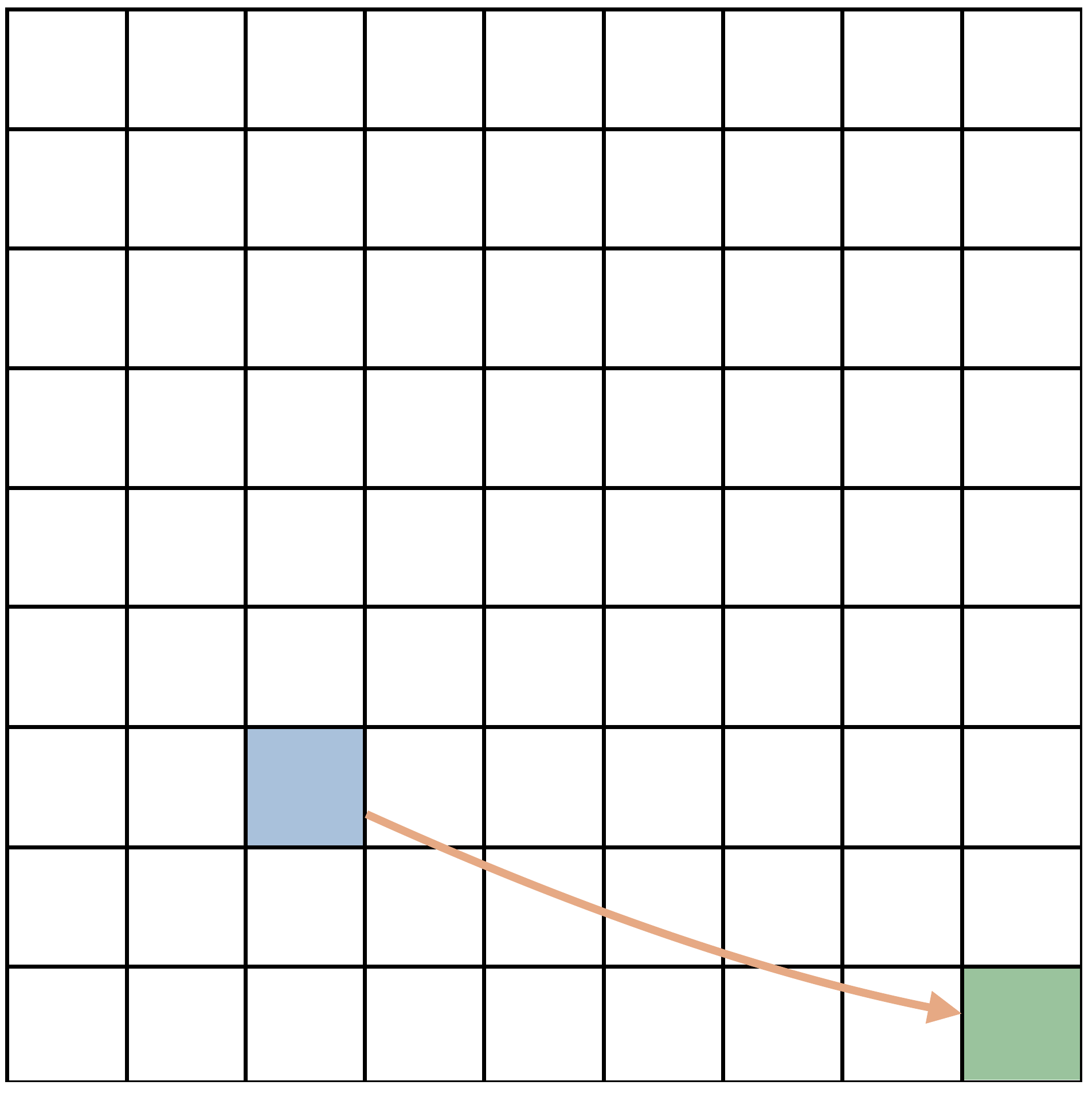}}\hspace{1mm}
    \subfloat[approx. $k = 2$]{\includegraphics[width=0.2\textwidth]{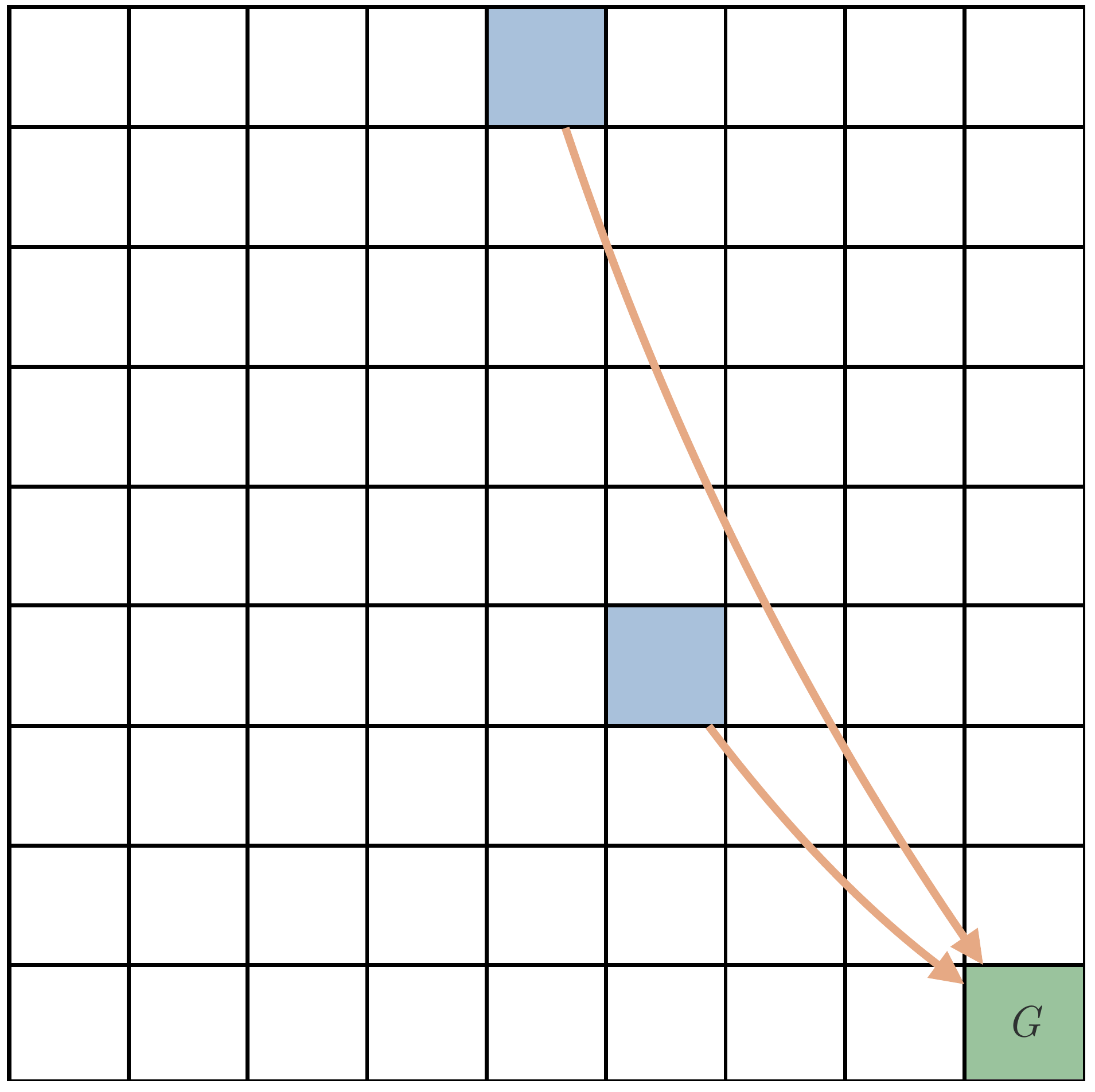}}\hspace{1mm}
    \subfloat[approx. $k = 3$]{\includegraphics[width=0.2\textwidth]{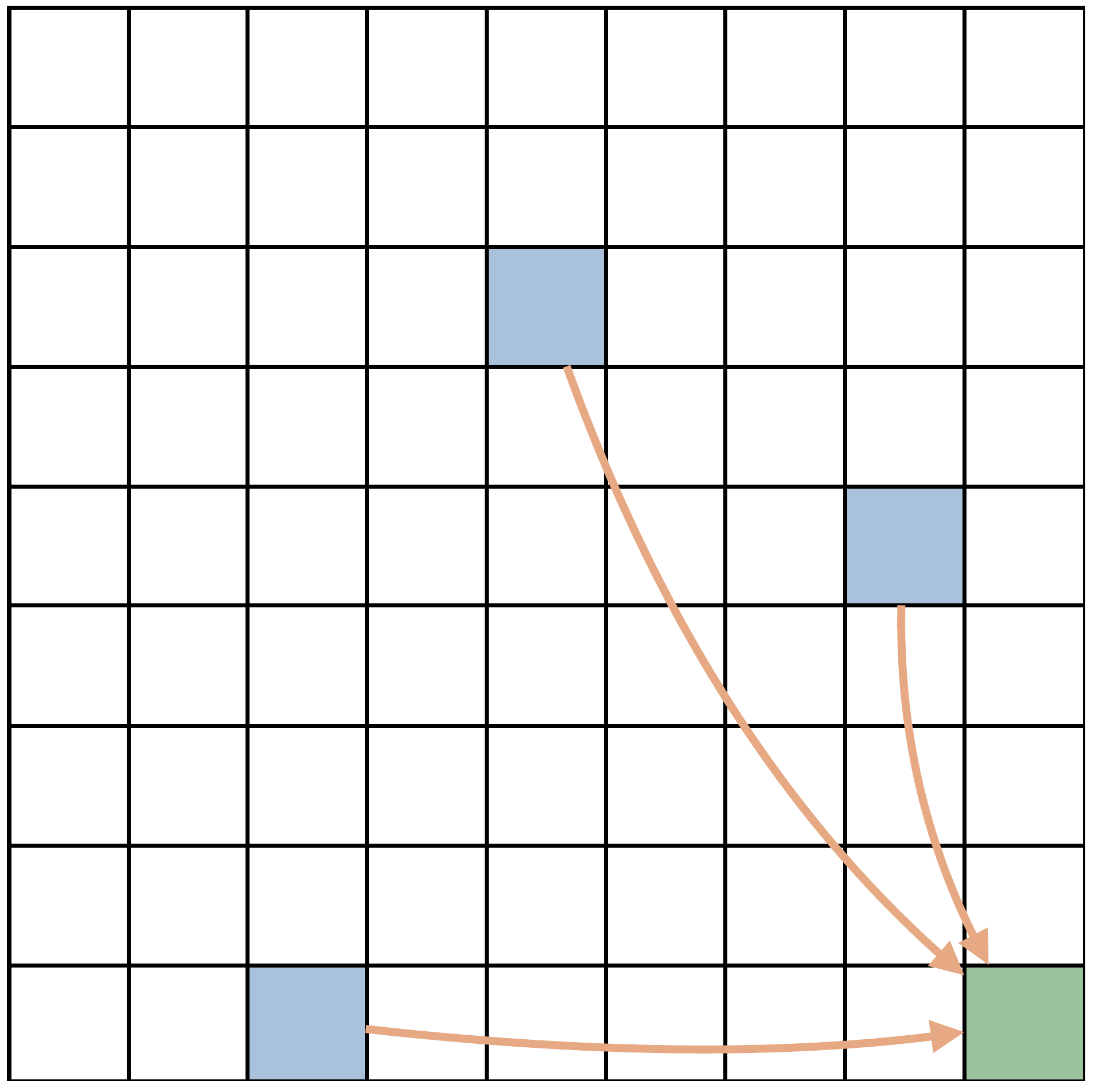}}
    
    \caption{Comparison of the optimal point options for planning vs. bottleneck options proposed for reinforcement learning in the four room domain. Initiating conditions are shown in blue, the goal in green. }
    \label{fig:9x9grid-viz}
\end{figure*}

\end{appendices}